\documentclass[journal]{IEEEtran}

\usepackage{amssymb}
\usepackage{palatino, epsfig, latexsym, algorithmic, epstopdf, amsmath, amssymb, color, amsthm, enumitem,multirow}
\usepackage{multirow,palatino, epsfig, latexsym, algorithmic, algorithm , epstopdf, amsmath, amssymb, color, amsthm,graphicx,enumitem}

\newcommand{\ignore}[1]{}

\newcommand{\keyw}[1]{\textit{\textbf{Index Terms--}}\textbf{#1}}

\newtheorem{convergenceTheorem}{Theorem}
\newtheorem{convergenceLemma}{Lemma}
\newtheorem{convergenceDefinition}{Definition}

\newtheorem{convergenceAssumption}{Assumption}

\begin{document}

\title{A theoretical guideline for designing an effective adaptive particle swarm}

\author{
	Mohammad Reza Bonyadi 
	\thanks{
		M. R. Bonyadi (m.bonyadi@uq.edu.au, rezabny@gmail.com) is with the Centre for Advanced Imaging (CAI), the University of Queensland, Brisbane, QLD 4072, Australia, and the Optimisation and Logistics Group, The University of Adelaide, Adelaide 5005, Australia. 
		
	}
}

\markboth{}%
{{Bonyadi}: A theoretical guideline for designing an effective adaptive particle swarm}

\maketitle
\IEEEpeerreviewmaketitle



\maketitle

\begin{abstract}
In this paper we theoretically investigate underlying assumptions that have been used for designing adaptive particle swarm optimization algorithms in the past years. We relate these assumptions to the movement patterns of particles controlled by coefficient values (inertia weight and acceleration coefficient) and introduce three factors, namely the autocorrelation of the particle positions, the average movement distance of the particle in each iteration, and the focus of the search, that describe these movement patterns. We show how these factors represent movement patterns of a particle within a swarm and how they are affected by particle coefficients (i.e., inertia weight and acceleration coefficients). We derive equations that provide exact coefficient values to guarantee achieving a desired movement pattern defined by these three factors within a swarm. We then relate these movements to the searching capability of particles and provide guideline for designing potentially successful adaptive methods to control coefficients in particle swarm. Finally, we propose a new simple time adaptive particle swarm and compare its results with previous adaptive particle swarm approaches. Our experiments show that the theoretical findings indeed provide a beneficial guideline for successful adaptation of the coefficients in the particle swarm optimization algorithm.
\end{abstract}
\keyw{Particle swarm optimization; covariance; correlation; stability}

\section{Introduction and motivation}
Particle swarm optimization (PSO) is a \cite{Smith} stochastic population-based optimization algorithm developed by \cite{PSOBaseKenn}. PSO has been applied to many optimization problems such as artificial neural network training and pattern classification~\cite{Engel05Fundamental,Poli08AnalPub}. Since 1995, different aspects of the algorithm, such as local convergence, invariance, stability, parameter setting, and topology have been investigated and many variants of the algorithm have been proposed \cite{bonyadi2016particle}. The movement pattern of particles in PSO, however, has only been investigated in a handful of articles \cite{OzcanSurfing1999,Trelea03Conv,Bonyadi2015,bonyadi2016particle}. Such analysis is very important for understanding why the algorithm performs well or fail for a given problem.

\textbf{Motivation:} In 2016, a comprehensive experimental study \cite{harrison2016inertia} showed that none of the adaptive (including time adaptive and self adaptive) approaches to control the inertia weight in PSO would work significantly better than the constant inertia weight proposed by \cite{Clerc02Explo} in a standard benchmark involving 60 standard test cases. This finding indicates that the underlying assumptions based on which these adaptive approaches have been designed does not hold for all search spaces, leading the adaptive methods to a poor performance comparing to a "good" constant value (e.g., proposed in \cite{Clerc02Explo}) for the coefficients. 

In this paper we relate this observation to the movement patterns of particles and investigate whether frequently used assumptions for adaptation are correct. In particular, we theoretically analyze the impact of coefficient values (inertia weight and acceleration coefficients) on the movement patterns of particles and relate those movement patterns to the local and global search abilities of particles. Our theoretical findings provide novel insights on fundamental assumptions to design effective adaptive coefficients that actually improve PSO in a more general setting\footnote{We only focus on PSO variants in which the coefficients have been investigated to improve the performance and other type of changes in the algorithm (e.g., hybridization with other methods, population size, see \cite{bonyadi2016particle} for other possible changes) are left out of the scope of this paper.}. We propose a simple time adaptive PSO based on our theoretical findings and compare our results with existing adaptive methods proposed in the literature (e.g., \cite{yang2015low,jiao2008dynamic,feng2007chaotic,chen2006natural,tanweer2015self,chauhan2013novel,Nickabadi11NovelPSO}).

Without loss of generality, this paper only considers minimization problems defined as follows:
\begin{equation}
\text{find } \vec{x}\in S \subseteq \mathbb{R}^d \text{ such that } \forall \vec{y}\in S, f(\vec{x}) \le f(\vec{y})
\end{equation}

\noindent where $ S $ is the \textit{search space} defined by $ \{x|l_i \le x_i \le u_i \text{ for all } i\} $, $ l_i $ and $ u_i $ are lower bound and upper bound of the values of the $ i^{th} $ dimension of $ S $, $ d $ is the number of dimensions, and $ f:\mathbb{R}^d \to \mathbb{R} $ is the objective function. The set of points that are generated by $ f(\vec{x}) $ for all $ \vec{x} \in S $ is called the landscape.

The structure of this paper is as follows. After a brief background on previous related works (section \ref{sec:background}), we show how the autocorrelation of positions of a particle (section \ref{sec:covcorr}), the expected movement distance (section \ref{sec:variance}), and the focus of the search (section \ref{sec:focussearch}) may characterize movement patterns of a particle. Then, we derive a system of equations (section \ref{sec:searchability}) that relates particle coefficients to some quantitative measure of these three factors of movement patterns. We provide the analytical solution of this system of equations that is in fact the coefficient values that guarantee achieving a given movement pattern. We finally propose a new adaptive approach and test its performance based on our theoretical findings.

\section{Background}
\label{sec:background}
In this section we provide a brief background information on early variants of PSO, existing studies on the variance of movement of particles in PSO, and different patterns of movements in PSO.

\subsection{Particle swarm optimization}
\label{Sec:PSOBasics}
Each particle in the Original PSO (OPSO)~\cite{PSOBaseEbr,PSOBaseKenn} contains three vectors: position ($ \vec{x}_t^{~i} $), velocity ($ \vec{v}_t^{~i} $), and personal best ($ \vec{p}_t^{~i} $). In OPSO, the $ j^{th} $ dimension of the position of a particle $ i $ is updated by
\begin{equation}
x_{t+1}^{i,j}=x_t^{i,j}+v_{t+1}^{i,j}
\label{Eq:positionUpdate}
\end{equation}
Dimension $ j $ of the velocity of particle $ i $, $ v_{t+1}^{i,j} $, is calculated by 
\begin{equation}
	v_{t+1}^{i,j}=v_t^{i,j}+\phi^{i,j}_{1,t}\left(p_t^{i,j}-x_t^{i,j}\right)+\phi^{i,j}_{2,t}\left(g_t^j-x_t^{i,j}\right)
	\label{Eq:velocOpso}
\end{equation}
\noindent where $ \phi^{i,j}_{1,t} $ and $ \phi^{i,j}_{1,t} $ are taken from two uniform random variables, $ \phi_1 $ and $ \phi_2 $, in $ [0, c_1] $ and $ [0, c_2] $ respectively (called \textit{acceleration coefficients}), and $ \vec{g}_t $ is the best personal best of the swarm. The particle that its personal best is $ \vec{g}_t $ is called the \textit{global best particle}. The vector $ \vec{p}_t^{~i} $ (position of the personal best of particle $ i $ at time $ t $) is updated by 
\begin{equation}
\vec{p}_{t+1}^{~i}=\left\{ 	
\begin{array}{*{35}{l}} 
\vec{x}_{t+1}^{~i} & f\left(\vec{x}_{t+1}^{~i}\right) < f\left(\vec{p}_{t}^{~i}\right)-\epsilon_0\text{ and }\vec{x}_{t+1}^{~i}\in S \\ 
\vec{p}_{t}^{~i} & \text{ otherwise } \\
\end{array} 	
\right.
\label{Eq:pUpdate}
\end{equation}
\noindent where $ \epsilon_0 $ is an arbitrarily small real value that represents the precision of the calculations. This constant can be set to the smallest possible value in the simulations. 

OPSO was studied by many researchers since 1995 and many new variants were proposed \cite{bonyadi2016particle}. The most frequently used variant of OPSO was proposed by \cite{Shi98ModPSO} (called IPSO, Inertia PSO, throughout this paper) in which velocity update rule was revised as follows:
\begin{equation}
	v_{t+1}^{i,j}=\omega^{i,j}_t v_t^{i,j}+\phi^{i,j}_{1,t}\left(p_t^{i,j}-x_t^{i,j}\right)+\phi^{i,j}_{2,t}\left(g_t^j-x_t^{i,j}\right)
	\label{Eq:VelocIPSO}
\end{equation}
\noindent where $ \omega^{i,j}_t $ is a constant real value called \textit{inertia weight} for all $ t $, $ i $, and $ j $. The main purpose for introduction of inertia weight was to enable balancing between global and local search ability of the particle \cite{Shi98ModPSO}.

\subsection{Variance and expectation of movement}
\label{sec:variancemovement}
Perhaps the first study\footnote{For theoretical analyses that investigate an arbitrary particle, the index $ i $ is ignored. Also, as the position and velocity of most PSO variants (including OPSO and IPSO) are updated for each dimension independently, any analyses conducted in one dimensional case is generalizable to multidimensional cases, that allows dropping the dimension index $ j $.} that investigated the variance of positions for a particle in IPSO was \cite{Jiang07StochConv} that was extended further by \cite{Poli09MeanAndVar}. These studies proved that if the variance of the positions of a particle converge to a fixed point\footnote{The fixed point of a sequence generated by a recursion $ z_{t+1}=f(z_t) $ is defined by $ \hat{z}=f(\hat{z}) $. If the recursion is convergent and continuous then there exists $ \hat{z} $ such that $ z_t \to \hat{z} $ when $ t $ grows.} then $ c<\frac{12(\omega^2-1)}{5\omega-7} $ where $c =c_1=c_2$. Both these articles conducted their analyses under the following assumption:
\begin{convergenceAssumption}
	\label{ass:simplifiedPSO}
	$ p_t=p_{t-1} $, $ g_t=g_{t-1} $, and $ \omega_t=\omega_{t-1} $ for all $ t $, and $ \phi_{1,t} $ and $ \phi_{1,t} $ are taken from two uniform random distributions, $ \phi_1 $ and $ \phi_2 $, in the interval $ [0, c1] $ and $ [0, c_2] $, respectively.
\end{convergenceAssumption}
This assumption is not quite realistic as $ p_t $ and $ g_t $ are not actually constant during a real run. Several articles tried to investigate particles under more realistic assumptions (interested readers are referred to \cite{bonyadi2016particle} for full discussion on this topic). \cite{Bonyadi2015Stagnation} modified this assumption as follows:
\begin{convergenceAssumption}
	\label{ass:genericRandomVariable}
	$ p_t $, $ g_t $, $ \omega_t $, $ \phi_{1,t} $, and $ \phi_{2,t} $ are taken from arbitrary random variables, $ p $, $ g $, $ \omega $, $ \phi_1 $, and $ \phi_2 $, with given expectation and variance for all $ t $.
\end{convergenceAssumption}
This assumption makes theoretical analyses independent of the objective function and the swarm size and topology \cite{Bonyadi2015Stagnation}. The only difference between a "swarm" of one particle and a swarm of multiple particles is that the information is shared through $ g_t $ (or a local best), i.e., $ g_t $ could be updated by other particles. As Assumption \ref{ass:genericRandomVariable} allows updating $ g_t $ randomly, the impact of other particles on the particle under analysis is already taken into account. In addition, from particles perspective, the only difference between objective functions is the way $ g_t $ and $ p_t $ are updated. Hence, random changes in $ g_t $ and $ p_t $ simulates the impact of the objective function as well. Hence, any conclusion extracted from this assumption about one single particle (e.g., stability, movement patterns) holds for any particle within a swarm of \textit{any arbitrary size and topology} \cite{Bonyadi2015Stagnation}\footnote{One should note that this assumption, although more realistic than Assumption \ref{ass:simplifiedPSO}, is still not the exact representative of PSO and there are still rooms for further improvements. However, this is the most realistic assumption made so far in the PSO literature, hence, it is considered in this article for all further analyses. One can find a more recent extension of this assumption in \cite{cleghorn2017particle}.}. 

The position update rule of PSO was represented by a stochastic recursion in \cite{Bonyadi2015Stagnation}, formulated by: 
\begin{equation}
x_{t+1}=lx_{t}-\omega x_{t-1}+{\phi_1}p+{\phi_2}g
\label{Eq:genericPSO}
\end{equation}
\noindent where $ l=1+\omega-\phi_1-\phi_2 $, $ p $, $ g $, $ \phi_1 $, $ \phi_2 $, and $ \omega $ are random variables with given expected values ($ \mu $) and standard deviations ($ \sigma $) generated at each iteration (Assumption \ref{ass:genericRandomVariable}). It was shown that the matrix form in Eq. \ref{eq:fullmatrixform} describes the movement of particles under the Assumption \ref{ass:genericRandomVariable}:
\begin{eqnarray}
	\vec{z}_{t+1}=M\vec{z}_t+\vec{b}
	\label{eq:fullmatrixform}
\end{eqnarray}
\noindent where $ \vec{z}_t=[E(x_t)~E(x_{t-1})~E(x_t^2)~E(x_{t-1}^2)~E(x_tx_{t-1})]^T $, $ \vec{b}=[E(P)~0~E(P^2)~0~0]^T $ ($ P={\phi_1}p_t+{\phi_2}g_t $), and 
\begin{equation*}
	M=\begin{bmatrix}
		E(l) & -E(\omega) & 0 & 0 & 0\\ 
		1 & 0 & 0 & 0 & 0\\ 
		2E(lP) & -2E(wP) & E(l^2) & E(\omega^2) & -2E(l\omega)\\ 
		0 & 0 & 1 & 0 & 0\\ 
		E(P) & 0 & E(l) & 0 & -E(\omega)
	\end{bmatrix}
\end{equation*}
In the matrix $ M $:
\begin{itemize}
	\item $ E(l)=1+\mu_\omega-\mu_{\phi_1}-\mu_{\phi_2} $
	\item $ E(\omega^2)=\sigma_\omega^2+\mu_\omega^2 $
	\item $ E({\phi_1}^2)=\sigma_{\phi_1}^2+\mu_{\phi_1}^2 $
	\item $ E({\phi_2}^2)=\sigma_{\phi_2}^2+\mu_{\phi_2}^2 $
	\item $ E(\omega P)=\mu_\omega(\mu_{\phi_1}\mu_p+\mu_{\phi_2}\mu_g) $
	\item $ E(l^2)=1+E(\omega^2)+E(\phi_1^2)+E(\phi_2^2)+2\mu_\omega-
	2(\mu_{\phi_1}+\mu_{\phi_2})-2\mu_\omega(\mu_{\phi_1}+\mu_{\phi_2})+2\mu_{\phi_1}\mu_{\phi_2} $
	\item $ E(P)=\mu_{\phi_1}\mu_p+\mu_{\phi_2}\mu_g $
	\item $ E(P^2)=E(p^2)E(\phi_1^2)+E(g^2)E(\phi_2^2)+2\mu_{\phi_1}\mu_{\phi_2}\mu_p\mu_g $
	\item $ E(lP)=\mu_{\phi_1} \mu_p+\mu_{\phi_2}\mu_g+\mu_w\mu_{\phi_1}\mu_p+\mu_\omega \mu_{\phi_2}\mu_g-\mu_pE(\phi_1^2)-\mu_{\phi_1}\mu_{\phi_2}(\mu_p+\mu_g)-\mu_gE(\phi_2^2) $
\end{itemize}

Although $ p_t $ and $ g_t $ are not constant during the run, the vector $ \vec{b} $ remains constant as it depends only on the mean and variance of $ p $ and $ g $. Analyses by \cite{Bonyadi2015Stagnation} showed that the eigenvalues of the matrix $ M $ are all independent of the mean and variance of $ p $ and $ g $. Hence, whether the variance of the positions of the particle is convergent is independent of how $ p $ and $ g $ are updated, under the Assumption \ref{ass:genericRandomVariable}. Using this matrix, it was proven that the necessary and sufficient conditions for the convergence of expectation of the particle positions is:
\begin{eqnarray}
-1 < \mu_\omega<1 \text{ and } 0<\mu_{\phi_1}+\mu_{\phi_2}<2(\mu_\omega+1)
\label{Eq:expectationConvergence}
\end{eqnarray}
Also, the fixed point of expectation was calculated as: 
\begin{eqnarray}
E_x=\frac{\mu_{\phi_1}\mu_p+\mu_{\phi_2}\mu_g}{\mu_{\phi_1}+\mu_{\phi_2}}
\label{Eq:expectationEquilibrium}
\end{eqnarray}
The fixed point of the variance was calculated \cite{Bonyadi2015Stagnation} by:
\begin{eqnarray}
	V_x=\hat{z}_3-\hat{z}_1^2=-\frac{k_3+k_4}{k_1k_2}(\mu_\omega + 1)
	\label{Eq:varianceFixPoint}
\end{eqnarray}
\noindent where $ \hat{z} $ is the fixed point of the recursion in Eq. \ref{Eq:genericPSO} and
\begin{itemize}
	\item $ k_1=(\mu_{\phi_1} + \mu_{\phi_2})^2 $,
	\item $ k_2=k_1(1-\mu_{\omega})+2(\mu_{\phi_1}+\mu_{\phi_2})(\mu_{\omega}^2+\sigma_{\omega}^2-1)+(\sigma_{\phi_1}^2+\sigma_{\phi_2}^2)(\mu_{\omega}+1) $, 
	\item $ k_3=k_1(\mu_{\phi_1}^2\sigma_p^2 + \mu_{\phi_2}^2\sigma_g^2 + \sigma_{\phi_1}^2\sigma_p^2 + \sigma_{\phi_2}^2\sigma_g^2) $,
	\item $ k_4=(\mu_{\phi_1}^2\sigma_{\phi_2}^2 + \mu_{\phi_2}^2\sigma_{\phi_1}^2)(\mu_g - \mu_p)^2 $.
\end{itemize}

It was proven that Eq. \ref{Eq:PSOConvergenceConditions} is a necessary condition for the convergence of the variance of particle's positions under the Assumption \ref{ass:genericRandomVariable}.
\begin{equation}
\begin{cases}
Condition~1: & -1<\mu_\omega<1\\ 
Condition~2: & 0<\mu_{\phi_1}+\mu_{\phi_2}<2(1+\mu_{\omega})\\
Condition~3: & k_2<0
\end{cases}
\label{Eq:PSOConvergenceConditions}
\end{equation}

It was experimentally shown that Eq. \ref{Eq:PSOConvergenceConditions} is also sufficient for the convergence of variance. Although the Assumption \ref{ass:genericRandomVariable} is still not the exact representation of particle movements, it is the most realistic model in the literature.

\subsection{Movement patterns}
\label{sec:patternmovement}
Different coefficient values may result in convergence or divergence of particles. These coefficients, however, impact another aspect of particles that is the movement pattern of particle positions \cite{Trelea03Conv,Bonyadi2015,BonyadiMovementPattern2017,bonyadi2016particle}. For example, a particle in IPSO with $ \omega=0.98 $ and $ c_1=c_2=0.198 $ oscillates smoothly while a particle with $ \omega=0.73 $ and $ c_1=c_2=1.64 $ oscillates more chaotically (see Fig. \ref{fig:IPSOBehaviourTime}). These oscillations might take place within a small range or a large range that is also dependent on the coefficients values. These patterns play important roles in the performance of the algorithm. For example, a particle that moves smoothly in the search space can be potentially more effective at the latter stages of the search process than a particle that jumps all over the search space~\cite{Bonyadi2015}. 

Despite the importance of these patterns, there have not been many studies focused on this topic \cite{bonyadi2016particle}. Most previous studies simply assumed that local or global search abilities of particles is a linear function of $ \omega $ \cite{Shi98ModPSO,zheng2003convergence}, an assumption that constituted a foundation for many adaptive approaches \cite{yang2015low,jiao2008dynamic,feng2007chaotic,chen2006natural,tanweer2015self,chauhan2013novel,Nickabadi11NovelPSO,Chatterjee06NonliomegaPSO,Zhan09AdaptivePSO} (see \cite{bonyadi2016particle,harrison2016inertia} for complete discussion). It was shown, however, that none of these approaches perform better than the constant coefficient proposed in \cite{Clerc02Explo} on a set of 60 benchmark functions \cite{harrison2016inertia}, that motivates us to investigate why is this the case.
\begin{figure}
	\centering
	\includegraphics[clip,width=0.5\textwidth]{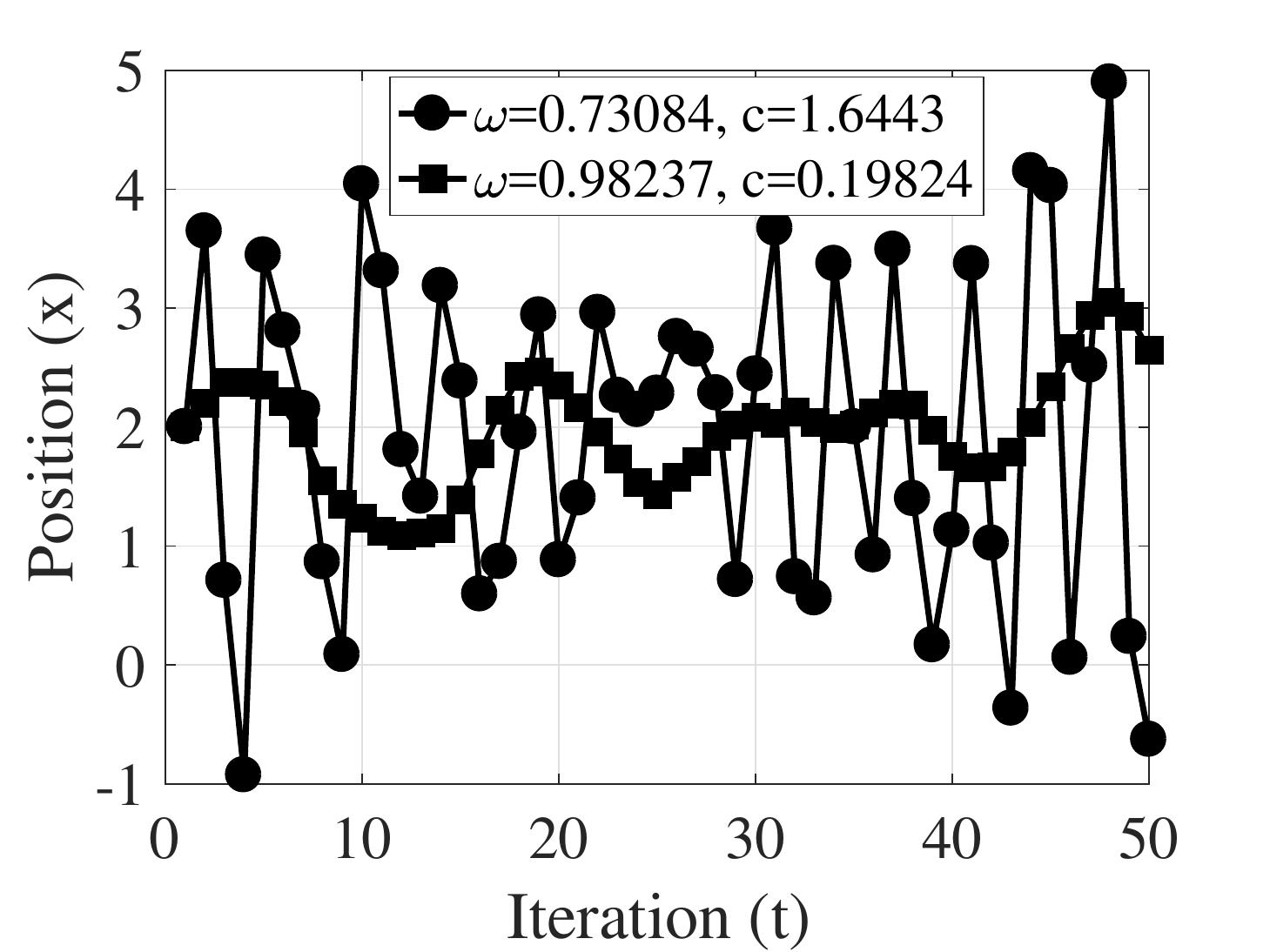}
	\caption{A particle in IPSO with $ \omega=0.98237 $ and $ c_1=c_2=0.19824 $ fluctuates smoothly while a particle with $ \omega=0.73084 $ and $ c_1=c_2=1.6443 $ oscillates more chaotically. Personal best and global best were updated randomly to simulate the impact of other particles and the objective function.}
	\label{fig:IPSOBehaviourTime}
\end{figure}

The trajectory of positions of a particle in OPSO was investigated by \cite{OzcanSurfing1999} where the update rules were simplified by replacing $ \phi_1 $ and $ \phi_2 $ by constants $ c_1 $ and $ c_2 $. The oscillation pattern and the magnitude of $ x_t $ were investigated for that simplified system and the effects of changing coefficients were visually illustrated. \cite{Trelea03Conv} categorized movement patterns of particles in IPSO into 4 groups (see Figure \ref{fig:TelereaPattern} (a)): non-oscillatory (particle position does not oscillate during the run), harmonic (particle position oscillates smoothly similar to a wave), zigzagging (particle position oscillate significantly at each iteration), and harmonic-zigzagging (combination of significant oscillation and wave-like oscillation). It was found that different patterns are observed by changing the values of coefficients, however, the patterns are not simple linear functions of the coefficients (see Figure \ref{fig:TelereaPattern}(a)). Both of these articles conduced their analyses under the Assumption \ref{ass:simplifiedPSO}.
\begin{figure*}
	\begin{tabular}{cc}
		\includegraphics[clip,width=0.5\textwidth]{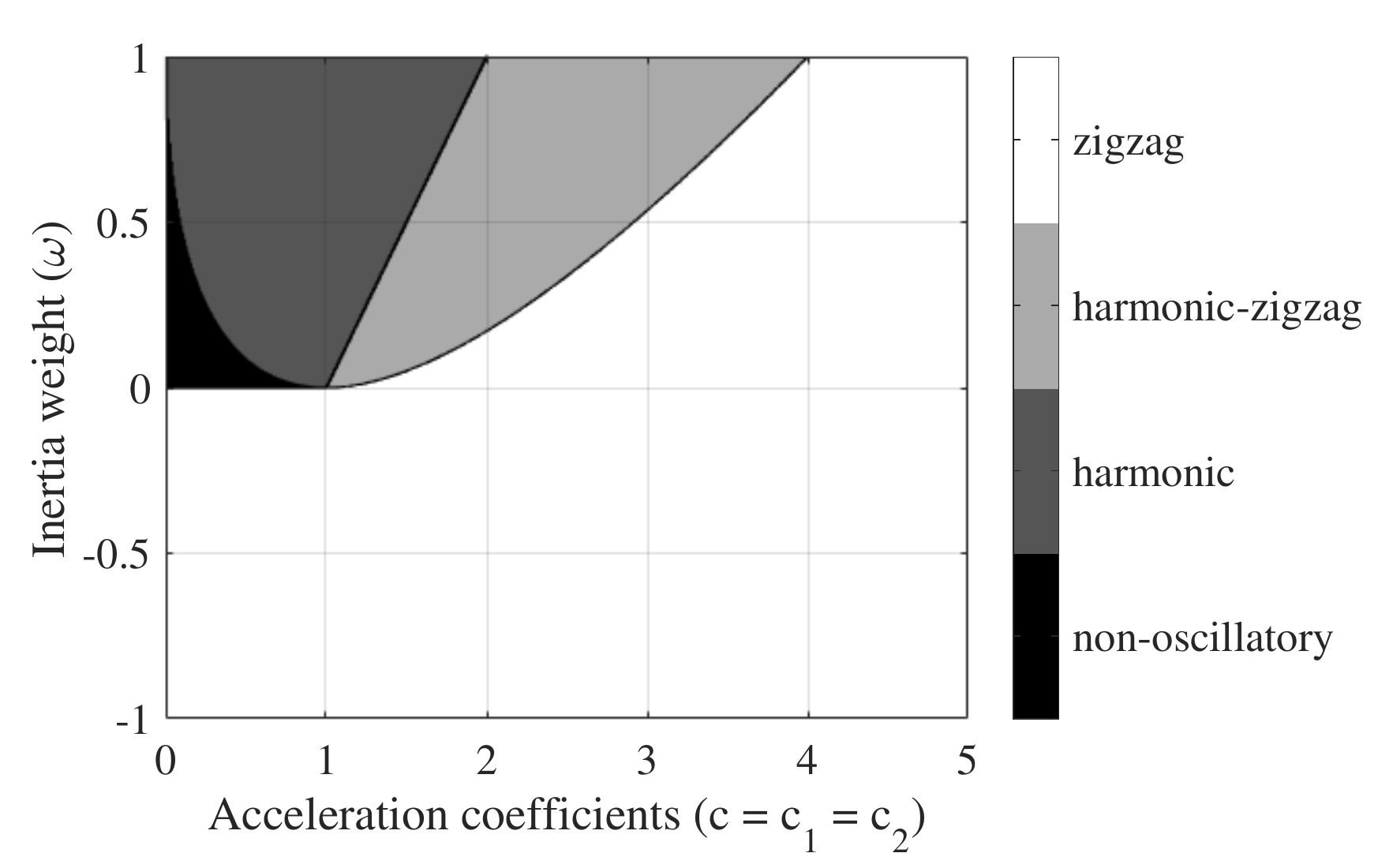} & \includegraphics[width=0.42\textwidth]{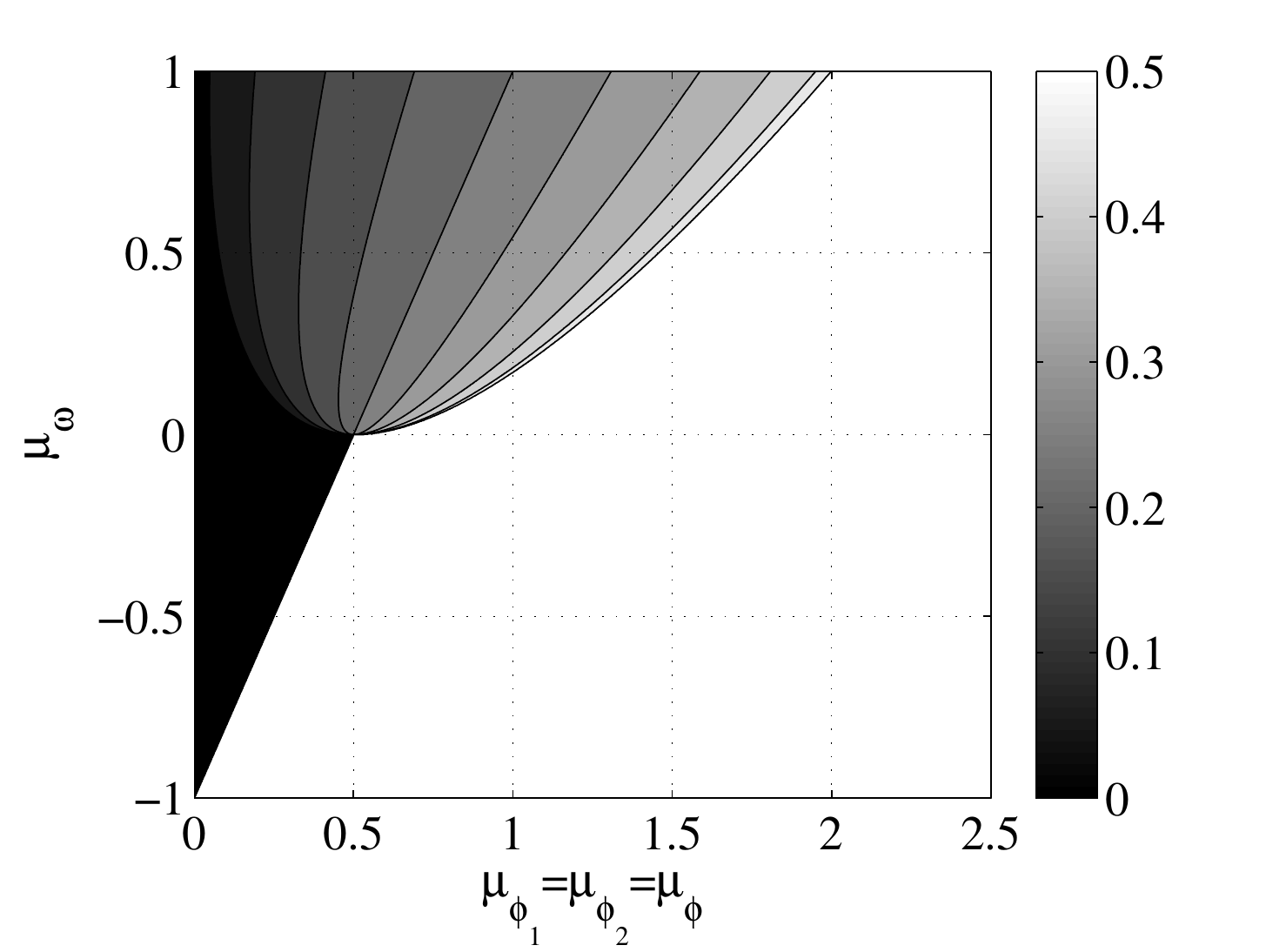} \\
		(a)  & (b) 
	\end{tabular}
	\centering	
	\caption{(a) \cite{Trelea03Conv} showed that the particles with different parameters may exhibit different movement patterns, in 4 categories. (b) \cite{BonyadiMovementPattern2017} represented these patterns by the movement frequency that is a non-linear function of coefficients.}
	\label{fig:TelereaPattern}
\end{figure*}

The impact of coefficients on the movement pattern of particles was investigated in \cite{Bonyadi2015} through some experiments. It was found that the base frequency of particle positions (the frequency of the largest amplitude among the Fourier series coefficients of the particle positions) has a direct relationship with the patterns of oscillation. This observation was used to estimate (based on some experiments) the boundaries corresponding to different oscillation patterns of IPSO and another PSO variant.

A recent article by \cite{BonyadiMovementPattern2017} investigated movement patterns of particles theoretically under the Assumption \ref{ass:genericRandomVariable}. They introduced two factors to characterize movement patterns: base frequency and the expected movement range. They formalized the base frequency introduced in \cite{Bonyadi2015} and used that as a measure of randomness for particle movements (see Figure \ref{fig:TelereaPattern} (b)). This measure generalized findings by \cite{Trelea03Conv} about movement patterns of the particles. The authors derived the relationship between the base frequency and the coefficients of PSO for $ \omega >0 $. They simplified the calculations for IPSO for the case where $ c=c_1=c_2 $. The variance was introduced as a measure for the range of movement and the relationship between this measure and coefficients was also calculated when $ c=c_1=c_2 $. Under these settings for IPSO ($ c=c_1=c_2 $ and $ \omega >0 $), a system of equations was introduced for which the solutions were the values of $ c $ and $ \omega $ to guarantee achieving a given base frequency and a range of movement. This system of equations was simplified to a degree-4 polynomial equation for which the solutions provided values for $ \omega $ to achieve a given base frequency and range of movement. Finally, it was shown that different coefficient values impose different base frequency and variance, a piece of information that can be used to optimize parameters to deal with various search spaces. After theoretical and experimental analyses, it was suggested to use $ c=1.711897 $ and $ \omega=0.711897 $ for applications.

\subsection{Adaptive and self adaptive coefficients}
\label{sec:bkg-adaptivecoefs}
There has been a large number of articles on adaptation of coefficients in PSO \cite{bonyadi2016particle,harrison2016inertia}. Most of adaptive approaches are based on the assumption that "increasing inertia weight leads the particle to a better global search", first introduced by \cite{Shi98Parameter}. Ideas presented in \cite{yang2015low,jiao2008dynamic,feng2007chaotic,chen2006natural} are examples of time-adaptive approaches based on this assumption and \cite{tanweer2015self,chauhan2013novel,Nickabadi11NovelPSO} are example articles in which self-adaptive approaches based on this assumption were proposed (see \cite{bonyadi2016particle,harrison2016inertia} for more examples). This means that decreasing the inertia weight during the iterations enables the particles to perform a better global search at the earlier stages and a better local search at the later stages of the searching process. In addition to this assumption, some of these articles (e.g., \cite{Nickabadi11NovelPSO,tanweer2015self}) further assumed that a larger $ \omega $ lead the particles to maintain their direction of movement that could be beneficial when the particle is improving. We investigate these two particular assumptions in this paper and provide theoretical evidence on whether or not they are correct in general.

\section{Characterization of movement patterns}
\label{sec:proposed}
We analyze a general formulation of PSO, introduced in Eq. \ref{Eq:genericPSO} in this paper. We also provide all analyses for IPSO, as an instance of PSO, defined as follows:
\begin{convergenceDefinition}
	IPSO is represented by a tuple $ <\omega, c, \alpha> $ through Eq. \ref{Eq:genericPSO} where both $ \phi_1 $ and $ \phi_2 $ follow the uniform distribution with $ \mu_{\phi_1}=\frac{c}{2} $, $ \mu_{\phi_2}=\alpha\frac{c}{2} $, $ c \in \mathbb{R} $, $ \alpha \in \mathbb{R} $, $ \sigma_{\phi_1}=\frac{c}{\sqrt{12}} $, and $ \sigma_{\phi_2}=\alpha\frac{c}{\sqrt{12}} $, and $ \omega $ is a constant ($ \sigma_\omega=0 $, $ \mu_\omega=\omega $).
	\label{def:IPSO}
\end{convergenceDefinition}
Based on this definition and convergence conditions in Eq. \ref{Eq:PSOConvergenceConditions}, the convergence conditions for IPSO$ <\omega, c, \alpha> $ are as follows:
\begin{equation}
\begin{cases}
Condition~1: & -1<\omega<1\\ 
Condition~2: & 0<c(1+\alpha)<4(1+\omega)\\
Condition~3: & k_2<0
\end{cases}
\label{Eq:IPSOConvergenceConditions}
\end{equation}

where $ k_2=\frac{c^2(3(1+\alpha)^2+(1+\alpha^2)(1+\omega))}{12} + c(1+\alpha)(\omega^2-1) $. From here on, whenever we use the term IPSO we refer to this definition and use the notation IPSO$ <\omega, c, \alpha> $ to specify the parameters. 

We use the Assumption \ref{ass:genericRandomVariable} in all of our analyses in this paper that is the most realistic assumption in literature (personal best and global best are updated) for theoretical analyses of PSO formulated by Eq. \ref{Eq:genericPSO}\footnote{It is important to note that this paper studies IPSO with this definition and any other types of PSO that do not fall into this definition (see \cite{bonyadi2016particle}) are out of the scope of this article.}. We first investigate factors that describe movement patterns in particles. Then, for each factor, we investigate how changing $ \omega $, $ c $, and $ \alpha $ would impact that factor in IPSO$ <\omega, c, \alpha> $. We then discuss how this is related to the adaptation strategies and assumptions made by previous articles.

We investigate three factors in this paper that characterize movements of a particle: relationship between the current and previous positions of a particle (section \ref{sec:covcorr}), expected movement distance for the particle (section \ref{sec:variance}), and the focus of the search towards personal best or global best (section \ref{sec:focussearch}). These factors together may represent the movement pattern of a particle to a large extent. We analyze these factors for the positions generated by Eq. \ref{Eq:genericPSO} and IPSO$ <\omega, c, \alpha> $. 

Although the first two factors were discussed in \cite{BonyadiMovementPattern2017}, that study includes major limitations:
\begin{itemize}
	\item Focusing on the base frequency that corresponds with the correlation between the current and the previous position of a particle only. We will prove that the particle positions at an iteration $ t $ could be correlated with not only its previous position but any other position before that. 
	\item Formulating the relationship between positions only for cases where $ \mu_{\omega} >0 $. It is, however, clear that $ \mu_{\omega} \le 0 $ is also a viable choice in PSO, that has not been investigated before.
	\item Formulating movement pattern in IPSO only for cases where $ c=c_1=c_2 $ and $ \omega>0 $. It is, however, clear that this is quite restrictive as any choice for $ c_1 $ and $ c_2 $ could be considered in PSO.
	\item Formulating the coefficients--movement patterns relationship as a degree-4 polynomial that, if used for controlling the coefficients, adds an overhead to the original calculations of PSO. 
	\item Ignoring the extend the particle focuses the search towards personal or global best. It is, however, important to change the focus of the search around personal best or global best in different stages of the search.
\end{itemize}
These limitations are addressed in this paper.  

\subsection{Relationship between current and previous positions in PSO: autocorrelation}
\label{sec:covcorr}
The relationship between two random variables can be formulated by the correlation between them. Pearson correlation, in particular, formulates the degree of linear dependency between two random variables (from here on, whenever we use the term "correlation" it refers to the Pearson correlation). In the context of time series, each sample can be considered as a random variable, dependent on the previous samples. The correlation between two consecutive samples in a time series then indicates to what extent a new sample can be predicted by a linear function of the previous sample. This can be generalized to the concept of the Autocorrelation that indicates to what extent a new sample, $ t+1 $, can be predicted by a linear function of a sample $ t-i+1 $, for any $ i>0 $.

The relationship between the positions of a particle can be formulated by their correlation. For example, the correlation between $ x_t $ and $ x_{t+1} $ generated by IPSO$<0.73084,1.6443,1>$ (the sequence shown by squares in Fig. \ref{fig:IPSOBehaviourTime}) is 0.047 while it is 0.897 for IPSO$ <0.98237,0.19824,1> $ (the sequence shown by circles in Fig. \ref{fig:IPSOBehaviourTime}) \footnote{This was calculated by simulating a particle with a given parameter set for a long run (1000 iterations in our experiment), shift the generated sequence by one sample, and calculate the correlation between the original and the shifted sequences. This was done when $ p $ and $ g $ where also random variables.}. This means that the positions at each iteration generated by a particle with the former settings could be expressed by a linear equation less accurately (more linear independence) than for a particle with the latter settings (less linear independence). The closer this correlation is to zero, the more linearly independent the consecutive positions are. Figure \ref{fig:corrExam} shows these correlations. 
\begin{figure*}[t]
	\centering
	\begin{tabular}{cc}
		\includegraphics[width=0.5\textwidth]{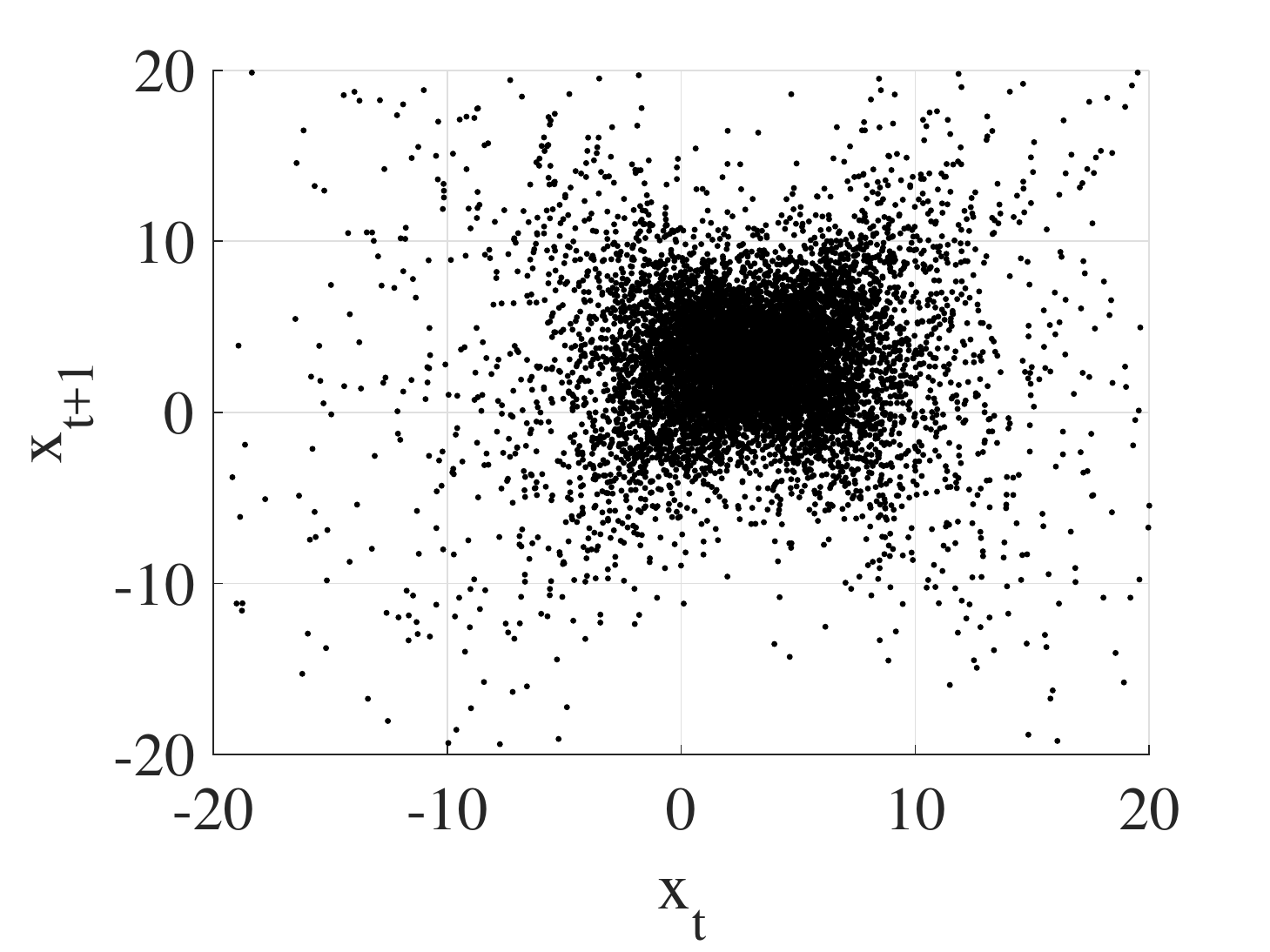} & \includegraphics[width=0.5\textwidth]{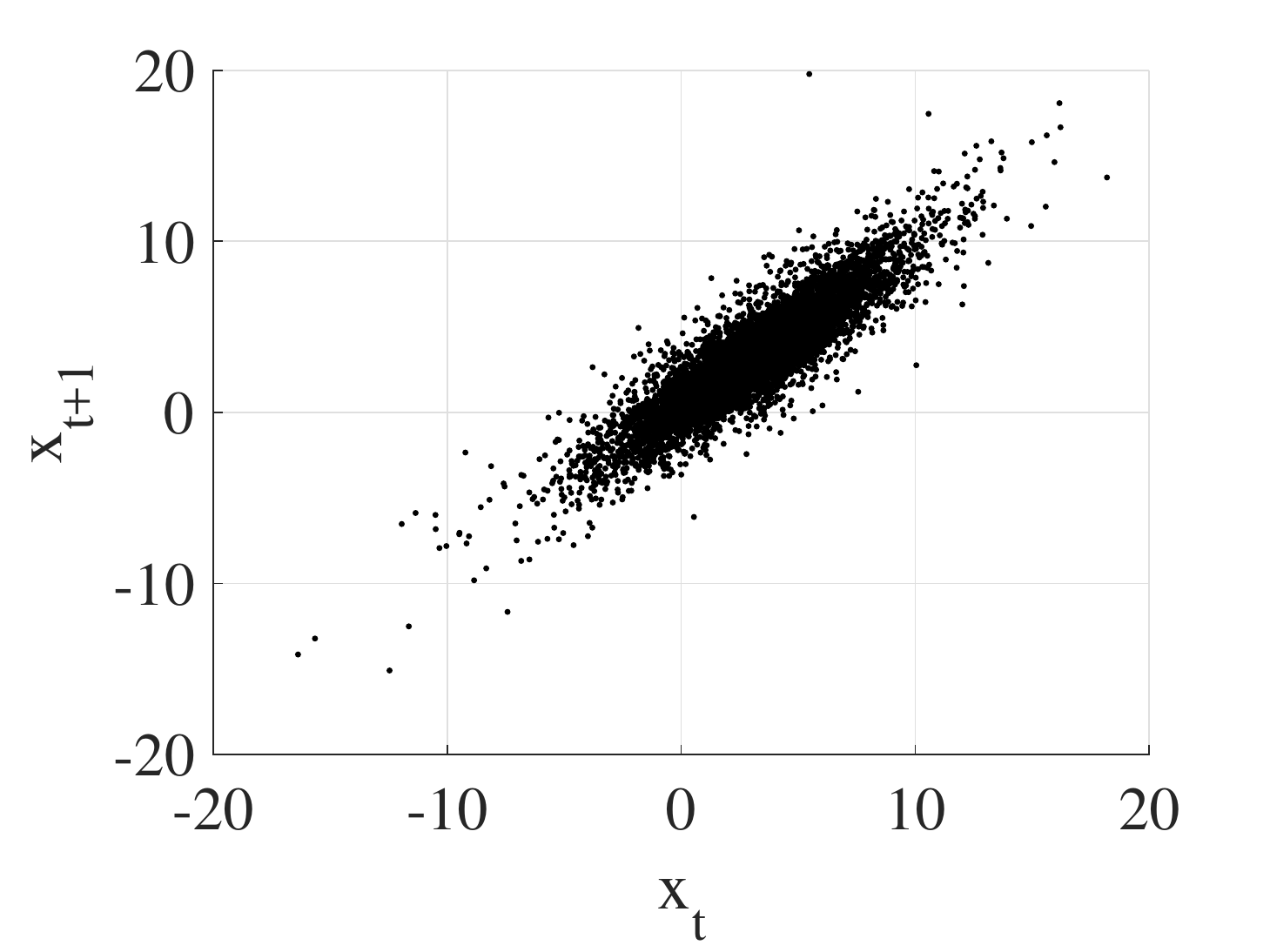} \\
		(a) & (b) \\
	\end{tabular}
	\caption{The relationship between $ x_t $ and $ x_{t+1} $ for IPSO (a) $ <0.73084,1.6443,1> $ and (b) $ <0.98237,0.19824,1> $. $ x_t $ and $ x_{t+1} $ in (a) are linearly independent.}
	\label{fig:corrExam}
\end{figure*}

The correlation between the position of a particle at iteration $ t+1 $ and $ t-i+1 $ is shown by $ \rho_i $ ($ \rho_0=1 $). If $ \rho_i $ is close to zero for all $ i>0 $ then $ x_{t+1} $ is linearly independent of all previous positions that lead to a random movement. We analytically derive the autocorrelation for the sequence of positions generated by Eq. \ref{Eq:genericPSO} at its equilibrium point. Let us start with the correlation between $ x_{t+1} $ and $ x_t $.
\begin{convergenceLemma}
	\label{lem:rho_1}
	The Pearson correlation between $ x_{t+1} $ and $ x_{t} $ generated by the stochastic recursion in Eq. \ref{Eq:genericPSO} at the equilibrium point is calculated by:
	\begin{equation}
	\label{Eq:rho_1}
	\rho_1=\frac{\mu_l}{\mu_\omega + 1}
	\end{equation}	
	where $ l=1 + \omega - \phi_1 - \phi_2 $ and $ \mu_l=E(l) $ (expectation of $ l $).
\end{convergenceLemma}
\begin{proof}
	The Pearson correlation ($ \rho $) between two random variables is calculated by $ \rho=\frac{Cov(a, b)}{\sigma(a)\sigma(b)} $, where $ \sigma(a) $ is the standard deviation of the random variable $ a $ and $ Cov(a,b) $ is the covariance between $ a $ and $ b $. The covariance between two random variables $ a $ and $ b $ is calculated by $ Cov(a,b)=E(ab)-E(a)E(b) $ where $ E(.) $ is the expectation operator. We calculate the covariance between $ x_{t+1} $ and $ x_t $ generated by the stochastic recursion in Eq. \ref{Eq:genericPSO} as $ Cov(x_{t+1},x_t)=E(x_{t+1}x_t)-E(x_{t+1})E(x_{t}) $. The fixed point of $ Cov(x_{t+1},x_t) $, shown by $ Cov_{t+1,t} $, can be calculated using the matrix form introduced in \cite{Bonyadi2015Stagnation} (see section \ref{sec:variancemovement}, Eq. \ref{eq:fullmatrixform}) as $ Cov_{t+1,t}=E(x_{t+1}x_t)-E(x_{t+1})E(x_t)=\hat{z}_5-\hat{z}_1\hat{z}_2 $. After simplifications, we calculate $ Cov_{t+1,t} $ as
	\begin{eqnarray}
	Cov_{t+1,t}= V_x(\frac{\mu_{l}}{\mu_\omega + 1})
	\end{eqnarray}	
	\noindent where $ V_x $ was introduced in section \ref{sec:variancemovement} and $ l=1 + \omega - \phi_1 - \phi_2 $. Recall that $ V_x $ is dependent on the moments of $ p $ and $ g $. When $ t $ is large we have $ \sigma(x_t)\sigma(x_{t+1})=\sigma^2(x_t)=V_x $, i.e. the variance of positions converges. Hence, the correlation between $ x_{t+1} $ and $ x_t $ at the equilibrium point is calculated by:
	\begin{equation}
	\label{Eq:correlation}
	\rho_1=\frac{\mu_l}{\mu_\omega + 1}
	\end{equation}
	\noindent where $ \rho_1 $ is the correlation between $ x_{t+1} $ and $ x_t $ and $ l=1 + \omega - \phi_1 - \phi_2 $. 	
\end{proof}
While \cite{BonyadiMovementPattern2017} formulated the relationship between $ x_t $ and $ x_{t+1} $ when $ \mu_\omega>0 $, the proposed equation in this paper, Eq. \ref{Eq:rho_1}, provides the correlation for any feasible value of $ \mu_{\phi_1} $,  $ \mu_{\phi_2} $, $ \mu_\omega $, and $ t $. Note also that this value is independent of $ p $ and $ g $, hence, the correlation remains constant even if the personal and global best move.

As it was discussed in section \ref{sec:variancemovement}, $ 0<\mu_{\phi_1}+ \mu_{\phi_2} < 2(\mu_\omega + 1) $ is a necessary condition for the convergence of expectation and variance of movement. If the coefficients are inside this boundary then the value of $ \rho_1 $ is a real value in $ (-1, 1) $ that indicates how positions of the particle at each iteration is related to its position at the previous iteration. If $ \rho_1=0 $ then there is no correlation between the position of the particle at each step and its previous position. This, however, does not provide any information about the correlation between $ x_{t+1} $ and $ x_{t-1} $. 
\begin{convergenceLemma}
	\label{lem:rho_2}
	The Pearson correlation between $ x_{t+1} $ and $ x_{t-1} $, shown by $ \rho_2 $, generated by the stochastic recursion in Eq. \ref{Eq:genericPSO} in the equilibrium point is calculated by:
	\begin{equation}
	\label{Eq:rho_2}
	\rho_2=\mu_l\rho_1-\mu_{\omega}\rho_0
	\end{equation}
	where $ \rho_0=1 $.
\end{convergenceLemma}
\begin{proof}
	The covariance between $ x_{t+1} $ and $ x_{t-1} $ is calculated by 
	\begin{eqnarray*}
	Cov_{t+1,t-1}=E(x_{t+1}x_{t-1})-E(x_{t+1})E(x_{t-1})
	\end{eqnarray*}
	As  $ x_{t+1} $, $ x_{t-1} $, and $ x_{t} $ are taken from the same distribution, we can assume that $ E(x_{t+1})=E(x_{t-1})=E(x_{t}) $, $ \sigma(x_{t+1})=\sigma(x_{t})=\sigma(x_{t-1}) $. From Eq. \ref{Eq:genericPSO}, we find that $ E(x_{t+1}x_{t-1})=E(l)E(x_tx_{t-1})-E(\omega)E(x_{t-1}^2)+E(P)E(x_{t-1}) $. Hence, after simplifications, 
	\begin{eqnarray*}
	Cov_{t+1,t-1}=E(l)E(x_tx_{t-1})-E(\omega)E(x_{t-1}^2)+\\E(P)E(x_{t-1})-E^2(x_{t})
	\end{eqnarray*}
	This means that $ Cov_{t+1,t-1} $ in the equilibrium point is equal to $ E(l)\hat{z}_5-E(\omega)\hat{z}_4+E(P)\hat{z}_2-\hat{z}_2^2 $. After simplifications:
	\begin{equation}
	\rho_2=\frac{Cov_{t+1,t-1}}{V_x}=1-2(\mu_{\phi_1}+\mu_{\phi_2})+\frac{(\mu_{\phi_1}+\mu_{\phi_2})^2}{\mu_{\omega}+1}
	\end{equation}
	This could be further simplified to $ \mu_l\rho_1-\mu_{\omega}\rho_0 $, where $ \rho_0=1 $ and $ l=1+\omega-\phi_1-\phi_2 $. 
\end{proof}

By definition, the value of $ \rho_0 $ is the correlation between the particle position and itself at each $ t $ that is expected to be $ 1 $. This Lemma indicates that the positions might not be correlated at every step, but at every second step. In particular, it is interesting to notice that
\begin{itemize}
	\item if $ \rho_1=1 $ then $ \mu_{\omega}=\mu_l-1 $, hence, $ \rho_2=1 $,
	\item if $ \rho_1=-1 $ then $ \mu_{\omega}=-\mu_l-1 $, hence, $ \rho_2=1 $,
	\item if $ \rho_1=0 $ then $ \mu_l=0 $, hence, $ \rho_2=-\mu_{\omega} $.
\end{itemize}

Therefore, ensuring $ \rho_1=0 $ does not result in a complete random behavior but only a random behavior in every step. A complete random search imposes zero correlation between all steps of the algorithm. Hence, we calculate the correlation between $ x_{t+1} $ and $ x_{t-i+1} $ for any $ i>1 $ as follows:
\begin{convergenceTheorem}
	\label{thr:rhos}
	\ignore{Assume that the coefficients $ \omega $, $ \mu_{\phi_1} $, and $ \mu_{\phi_2} $ guarantee convergence of expectation (Eq. \ref{Eq:expectationConvergence}).} The Pearson correlation between $ x_{t+1} $ and $ x_{t-i+1} $, denoted by $ \rho_i $, generated by the stochastic recursion in Eq. \ref{Eq:genericPSO} at the equilibrium point is calculated by:
	\begin{equation}
		\rho_i=\mu_l\rho_{i-1}-\mu_\omega\rho_{i-2}
	\end{equation}
	for any $ i>1 $, where $ l=1+\omega-\phi_1-\phi_2 $, $ \rho_1=\frac{\mu_l}{\mu_{\omega}+1} $, and $ \rho_0=1 $.
\end{convergenceTheorem}
\begin{proof}
	The value for $ Cov(t+1,t-i+1) $ (denoted by $ Cov_i $) is calculated by 
	\begin{eqnarray*}
	 Cov_i=Cov(t+1,t-i+1)= \\E(x_{t+1}x_{t-i+1})-E(x_{t+1})E(x_{t-i+1})=\\ 
	 E(lx_tx_{t-i+1}-\omega x_{t-1}x_{t-i+1}+Px_{t-i+1})- \\	 E(x_{t+1})E(x_{t-i+1})=E(l)E(x_tx_{t-i+1})-\\
	 \mu_\omega E(x_{t-1}x_{t-i+1})+ E(P)E(x_{t-i+1})-E(x_{t+1})E(x_{t-i+1})
	\end{eqnarray*}
	As $ t $ is large, $ E(x_{t})=E(x_{t+j}) $ for any $ j $ ($ j $ is much smaller than $ t $), $ E(x_tx_{t-i+1})=E(x_{t+1}x_{t-i+2}) $, and $ E(x_{t-1}x_{t-i+1})=E(x_{t+1}x_{t-i+3}) $. Also, 
	\begin{eqnarray*}
	Cov(t+1,t-i+1)=	E(l)E(x_tx_{t-i+1})-\\
	\mu_\omega E(x_{t-1}x_{t-i+1})+E(P)E(x_t)-E^2(x_t)=\\
	E(l)\left(E(x_{t}x_{t-i+1})-E(x_{t})(E(x_{t-i+1})+E(x_{t-i+1}))\right)-\\
	\mu_\omega\left(E(x_{t}x_{t-i+2})-E(x_{t})(E(x_{t-i+2})+E(x_{t-i+2}))\right)+\\
	E(P)E(x_t)-E^2(x_t)=E(l)Cov_{i-1}-\mu_\omega Cov_{i-2}+\\
	E(l)E^2(x_{t+1})-\mu_\omega E^2(x_{t})+E(P)E(x_t)-E^2(x_t)=\\
	E(l)Cov_{i-1}-\mu_\omega Cov_{i-2}+	E(P)E(x_t)-\\
	E^2(x_{t})(1-E(l)+\mu_\omega)
	\end{eqnarray*}
		
	After simplifications, we get
	\begin{eqnarray*}
		\rho_i=\frac{Cov(t+1,t-i+1)}{V_x}=	\mu_l\rho_{i-1}-\mu_\omega \rho_{i-2}+\\\frac{E(P)E_x-(1-E(l)+\mu_\omega)E^2_x}{V_x}
	\end{eqnarray*}
	It is interesting to see that 
	\begin{eqnarray*}
		 E(P)E_x-(1-E(l)+\mu_\omega)E^2_x=\\ (\mu_{\phi_1}p+\mu_{\phi_2}g)E_x-(\mu_{\phi_1}+\mu_{\phi_2})E^2_x=\\
		 (\mu_{\phi_1}p+\mu_{\phi_2}g)\frac{\mu_{\phi_1}p+\mu_{\phi_2}g}{\mu_{\phi_1}+\mu_{\phi_2}}-\\
		 (\mu_{\phi_1}+\mu_{\phi_2})\frac{(\mu_{\phi_1}p+\mu_{\phi_2}g)^2}{(\mu_{\phi_1}+\mu_{\phi_2})^2} =0
	\end{eqnarray*}
	Hence
	\begin{eqnarray*}
		\rho_i=\mu_l\rho_{i-1}-\mu_\omega \rho_{i-2}
	\end{eqnarray*}	
	that completes the proof. $ \rho_1 $ is calculated by Eq. \ref{Eq:rho_1} and $ \rho_0=1 $.
\end{proof}

According to Lemma \ref{lem:rho_1} and Theorem \ref{thr:rhos}, the correlation between $ x_{t+1} $ and $ x_{t-i+1} $, $ i \ge 0 $, shown by $ \rho_i $, for the recursion in Eq. \ref{Eq:genericPSO} is given by:
\begin{eqnarray}
\label{Eq:correlationGeneric}
\rho_i=
\begin{cases}
1 & i=0\\
\frac{\mu_l}{\mu_\omega + 1} & i=1\\
\mu_l\rho_{i-1}-\mu_\omega \rho_{i-2} & i>1
\end{cases}
\end{eqnarray}

\noindent where $ l=1+\omega-\phi_1-\phi_2 $. This indicates that if $ \rho_1=\rho_2=0 $ then $ \rho_i=0 $ for any $ i>2 $, hence, PSO performs as a pure random search method as there will be no correlation between each step and any of the previous steps. According to Eq. \ref{Eq:rho_1} and \ref{Eq:rho_2}, this takes place only if $ \mu_{\phi_1}+\mu_{\phi_2}=1 $ and $ \mu_\omega=0 $. If we set $ \mu_{\phi_1}+\mu_{\phi_2}=1 $ and $ \mu_\omega=0 $ in Eq. \ref{Eq:genericPSO}, we can observe that $ x_{t+1} $ is just a function of $ \phi_1p_t+\phi_2g_t $ that is a random point generated according to $ p $, $ g $, $ \phi_1 $ and $ \phi_2 $ that, based on the Assumption \ref{ass:genericRandomVariable}, are all random variables. 

For IPSO $ <\omega,c,\alpha> $, the value of $ \rho_i $ for all $ i>0 $ is calculated by:
\begin{eqnarray}\label{Eq:IPSOallRhos}
\rho_i=
\begin{cases}
1 & i=0\\
\frac{\mu_l}{\omega + 1} & i=1\\
\mu_l\rho_{i-1}-\omega \rho_{i-2} & i>1
\end{cases}
\end{eqnarray}
\noindent where $ \mu_l=1+\omega-\frac{c}{2}-\frac{\alpha c}{2} $. This means that IPSO $ <0, 1, 1> $ is the only complete random search as only this setting leads to $ \rho_i = 0 $ for all $ i > 0 $. Any other settings for IPSO lead to a movement pattern in which the position of each particle at each time point is correlated with its position at some other previous time point, formulated by Eq. \ref{Eq:IPSOallRhos}.

In order to test to what extent these findings are in agreement with the actual PSO, we calculated the value of $ \rho_i $ ($ i \in \{1,2,...,20\} $) for IPSO $ <0.73084 ,1.6443,1> $ (the sequence showed by circles in Fig. \ref{fig:IPSOBehaviourTime}) and $ <0.98237,0.19824,1> $ (the sequence showed by squares in Fig. \ref{fig:IPSOBehaviourTime}). We tested two settings to update personal and global best: 1) $ p $ and $ g $ were random variables (uniform in $ [-9,11] $ and $ [-5,15] $, respectively), 2) $ p $ and $ g $ were random walks with $ p_{t+1}=p_t+\frac{r_{1,t}}{t} $ and $ g_{t+1}=g_t+\frac{r_{2,t}}{t} $ where $ r_{1,t} $ and $ r_{2,t} $ are uniform random numbers in $ [-1,1] $. The first setting, $ p $ and $ g $ are random variables, generates most unstructured sequence of values for $ p $ and $ g $, simulating a very rugged search space with many local optima and a swarm of particles. The second setting, $ p $ and $ g $ being random walks with decreasing walk length in iterations, simulates usual optimization scenarios in which $ p $ and $ g $ move more at the early stages of the search while move less at the latter stages within the swarm. Figure \ref{fig:Autocorr} shows the results. It is clear that the calculated correlations (using Eq. \ref{Eq:IPSOallRhos}) and simulated values are very close to one another. Figure \ref{fig:Cor} shows $ \rho_1 $ for different $ \alpha $, $ c $, and $ \omega $ for IPSO settings.

\begin{figure*}[t]
	\centering
	\begin{tabular}{cc}
		\includegraphics[width=0.5\textwidth]{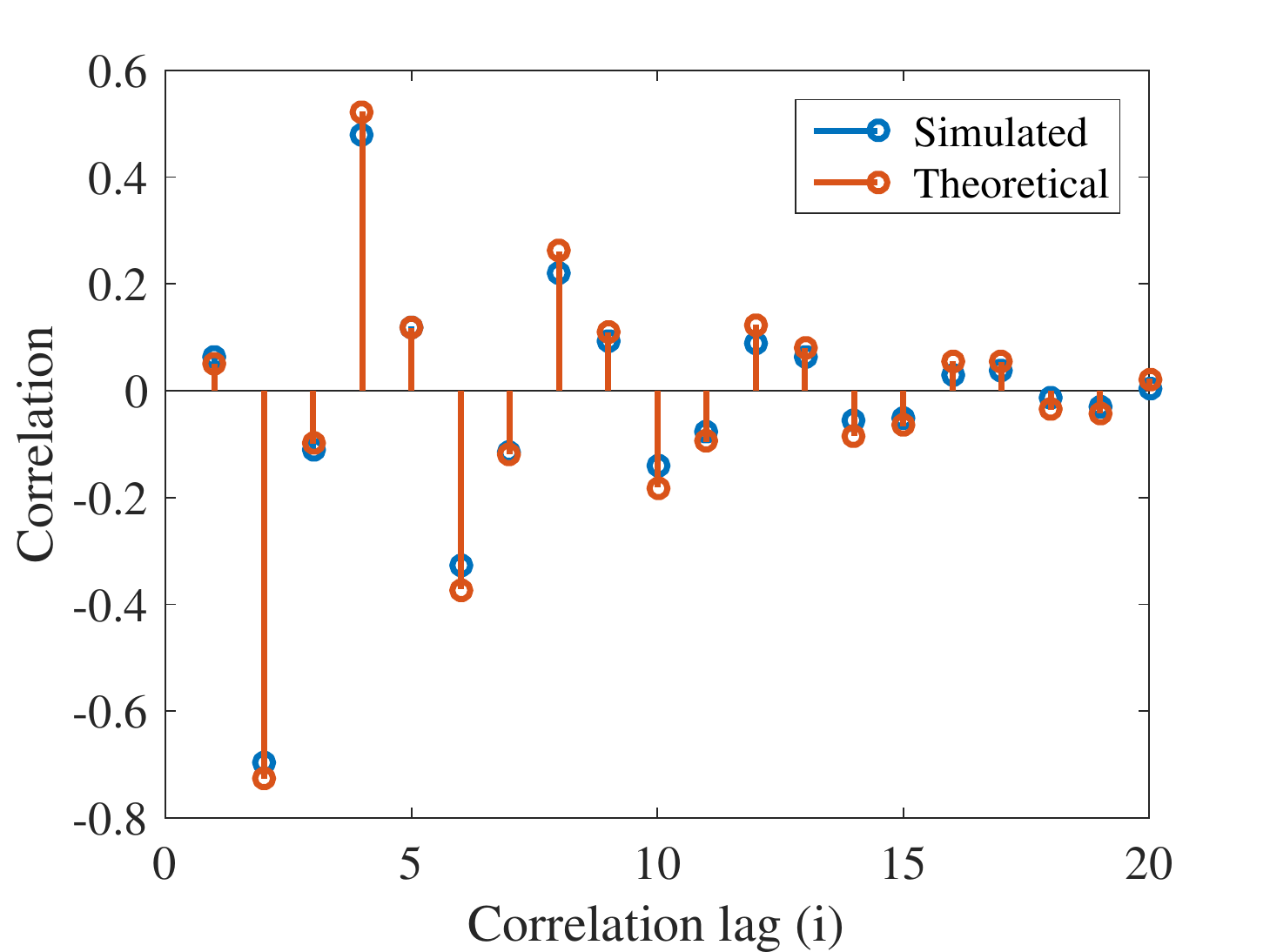} & \includegraphics[width=0.5\textwidth]{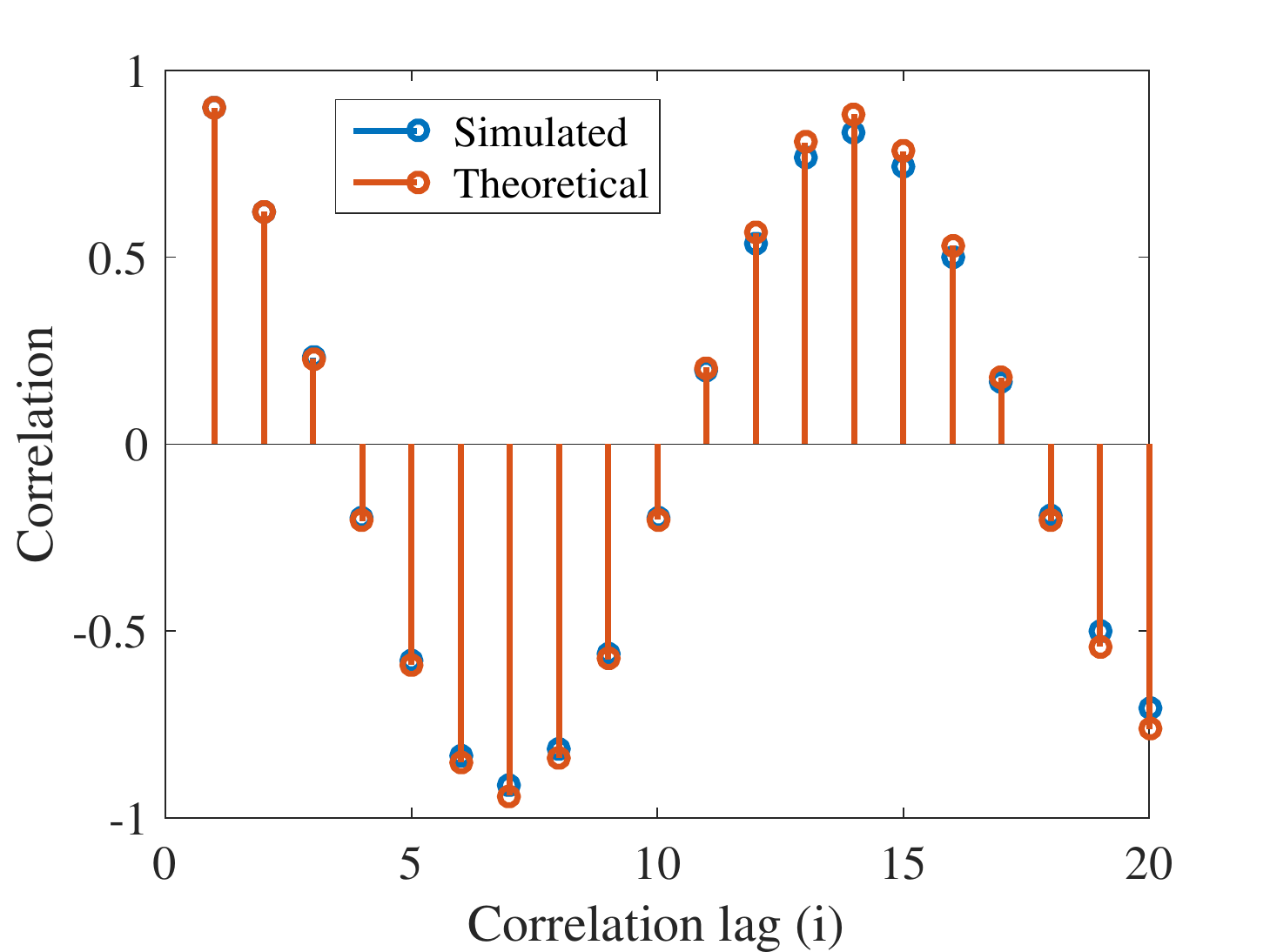} \\
		(a)  & (b) \\
	\end{tabular}
	\caption{Autocorrelation with maximum time lag of 20 calculated both analytically and experimentally for IPSO (a) $ <0.73084,1.6443,1> $ and (b)   $ <0.98237,0.19824,1> $.}
	\label{fig:Autocorr}
\end{figure*}

\begin{figure*}[t]
	\centering
	\begin{tabular}{cc}
		\includegraphics[width=0.5\textwidth]{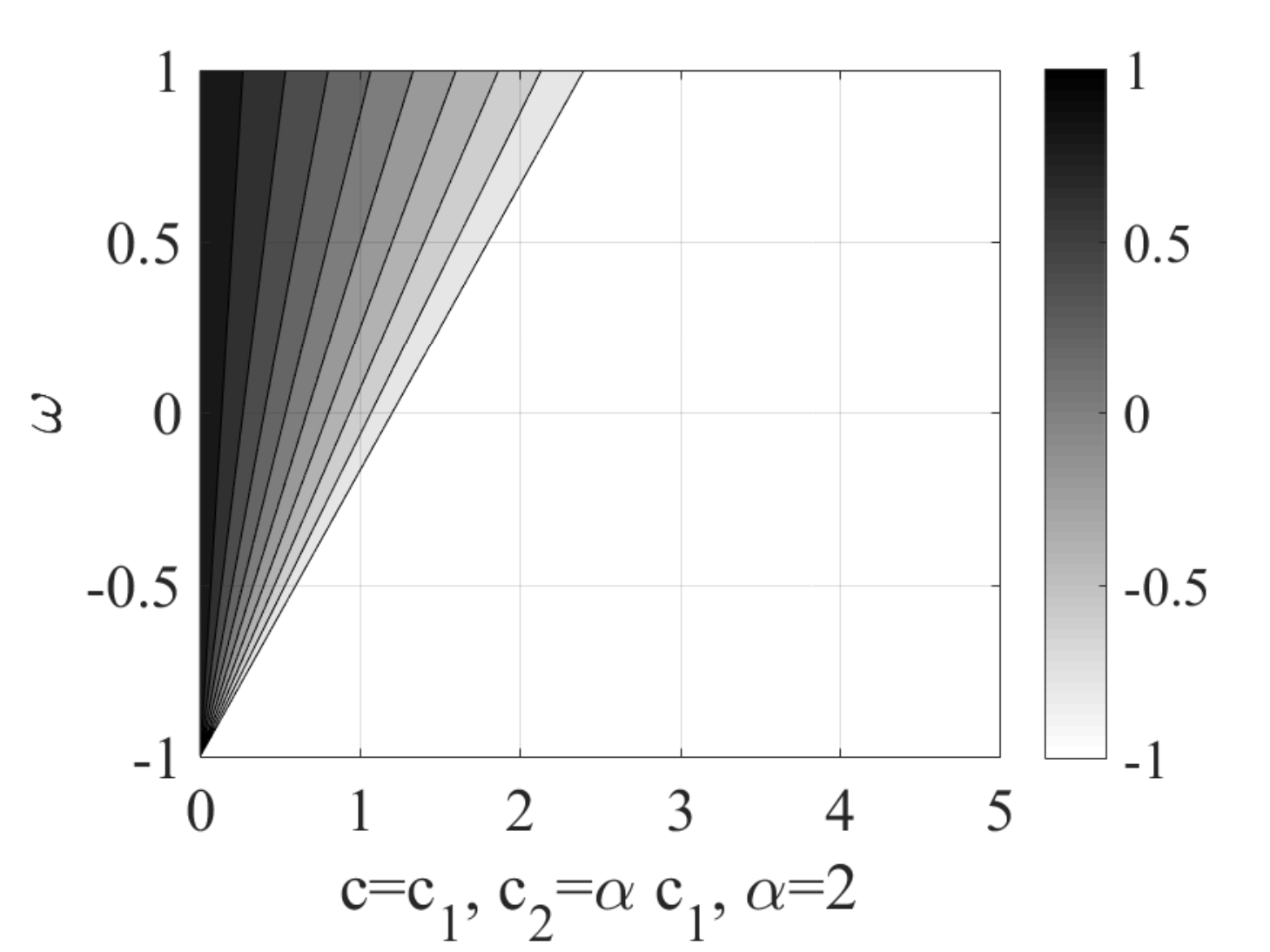} & \includegraphics[width=0.5\textwidth]{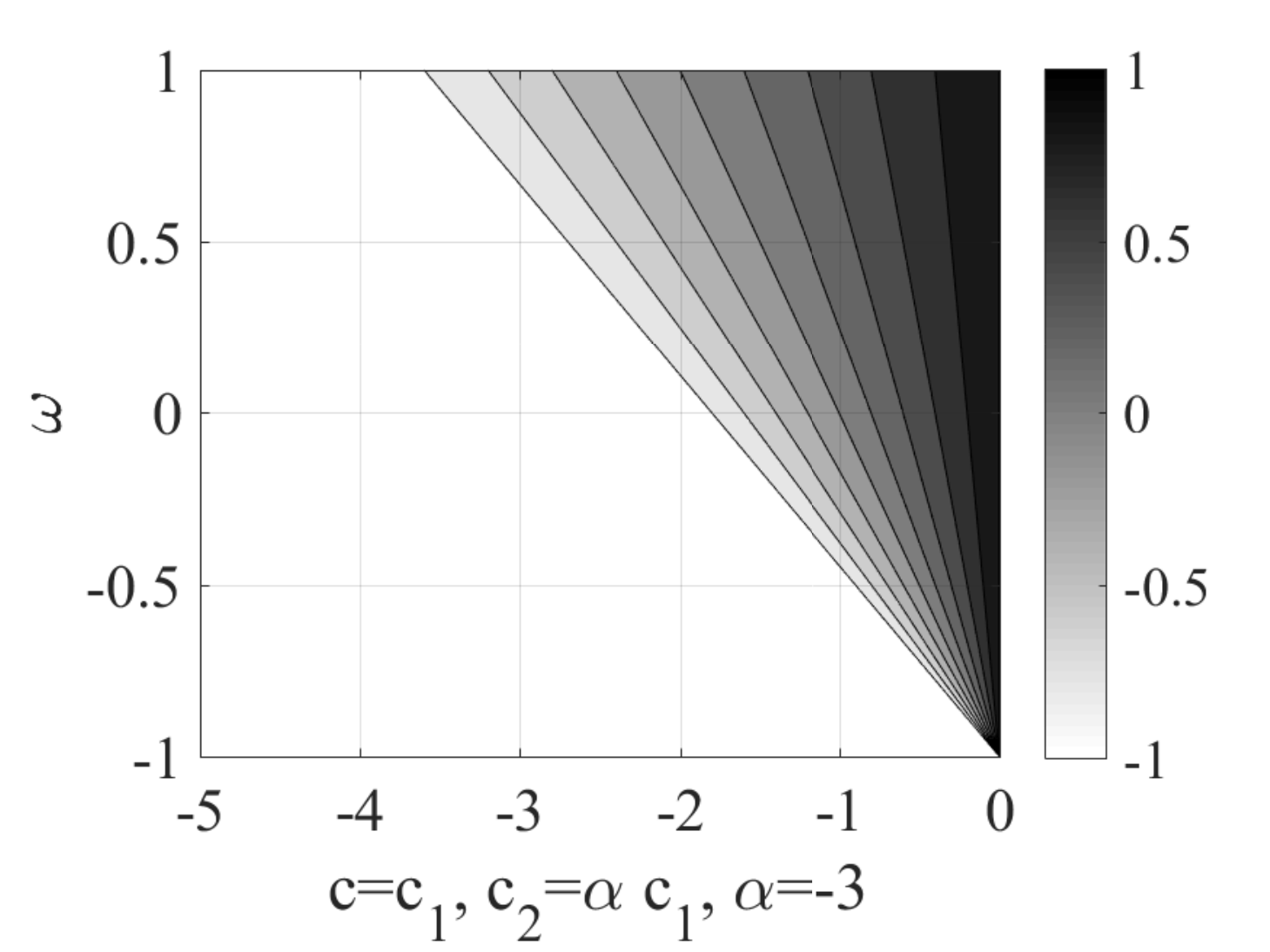} \\
		(a) & (b) \\
	\end{tabular}
	\caption{The value of $ \rho_1 $ when (a) $ \alpha=2 $ and (b) $ \alpha=-3 $.}
	\label{fig:Cor}
\end{figure*}

\subsection{Expected movement distance and expected search range}
\label{sec:variance}
The dependency between positions of a particle in different iterations is formulated by the autocorrelation function, provided in the Theorem \ref{thr:rhos}. This function, however, does not indicate how close the consecutive positions are. In fact, it is possible that the consecutive positions are tightly correlated while they are still far from one another, or they are linearly independent while close to one another. Hence, we investigate the expected distance that a particle moves at each step as another important factor of the movement pattern of the particle. 
\begin{convergenceDefinition}
	The expected movement distance for a particle at iteration $ t+1 $ is formulated by $ E(d_{t+1})=E([x_{t+1}-x_{t}]^2) $ that is equal to $ E(v_{t+1}^2) $. 
\end{convergenceDefinition}

We calculate the expected movement as follows:
\begin{convergenceLemma}
	\label{lem:movementDistance}
	The expectation of movement distance, $ E(d_{t+1}) $ ($ D_x $ for short), at the equilibrium point is formulated by 
	\begin{equation}
	\label{Eq:velocityAmp}
	D_x=E(d_{t+1})=2V_x(1-\rho_1)
	\end{equation}
\end{convergenceLemma}
\begin{proof}
	One can simply see that:
	\begin{eqnarray*}
		E(d_{t+1})=E([x_{t+1}-x_{t}]^2)=E(v_{t+1}^2)=\\
		E(x_{t+1}^2)+ E(x_{t}^2)-2E(x_{t}x_{t+1})
	\end{eqnarray*}	
	Because the calculation is performed in the equilibrium point, we can assume that $ E(x_{t+1}^2)= E(x_{t}^2) $, hence
	\begin{eqnarray*}
		E(v_{t+1}^2)=2(E(x_{t+1}^2)-E(x_{t+1}x_{t}))=\\
		2(E(x_{t+1}^2)-E^2(x_{t+1})+E^2(x_{t+1})-E(x_{t+1}x_{t}))=\\
		2(V_x-Cov_1)=2(V_x-V_x\rho_1)=2V_x(1-\rho_1)
	\end{eqnarray*}
	that completes the proof.
\end{proof}

\ignore{
This Lemma shows that the expected movement distance (i.e., the average velocity) is a function of $ V_x $ and $ \rho_1 $. Therefore, it is possible to achieve a small expected movement distance while $ V_x $ is large (i.e., a large $ V_x $ with a large $ \rho_1 $). This is correspondence to a movement pattern in which the consecutive positions are close to one another while the area the particle covers during the search is large. 
}

Because $ D_x $ is a function of $ V_x $ and $ \rho_1 $, we investigate $ V_x $ in the rest of this section instead of $ D_x $. Let us first describe the conceptual relation between $ V_x $ and the movement pattern.

A particle in PSO oscillates around the fixed point of its expectation ($ E_x=\frac{\mu_{\phi_1}\mu_p+\mu_{\phi_2}\mu_g}{\mu_{\phi_1}+\mu_{\phi_2}} $) during iterations. This oscillation enables the particle to reach new points in the search space and test if they are better than what has been found so far. The expectation of the distance that the position of a particle moves away from its fixed point (called the \textit{expected search range}) can be formulated by $ E[(x_t-E_x)^2] $ that is equal to the variance of the positions, $ V(x) $. When $ t $ is large, this expected search range can be characterized by the fixed point of the variance of the particle positions, $ V_x $. In other words, $ V_x $ is a measure of the range the particle "covers" during a large number of iterations. One interpretation for the expected search range is the ability of the particle to perform local or global search, i.e., the smaller the expected search range is, the better the particle performs a local search. 

The variance of positions generated by Eq. \ref{Eq:genericPSO} has been investigated in details by \cite{Bonyadi2015Stagnation} for IPSO but only when $ c=c_1=c_2 $, that is equivalent to IPSO $ < \omega , c, 1> $. We investigate the expected search range for IPSO defined by Definition \ref{def:IPSO} for an arbitrary $ < \omega , c, \alpha> $ that is a more generic setting.

The variance fixed point, $ V_x $, is calculated by Eq. \ref{Eq:varianceFixPoint}. For IPSO, $ k_1 $, $ k_2 $, $ k_3 $, and $ k_4 $ are written as:
\begin{itemize}
	\item $ k_1=\frac{c^2}{4}(1 + \alpha)^2 $
	\item $ k_2=k_1(1-\omega)+c(1+\alpha)(\omega^2-1)+\frac{c^2}{12}(1+\alpha^2)(\omega+1) $,
	\item $ k_3=k_1\frac{c^2}{3}(\sigma_p^2 + \alpha^2\sigma_g^2) $,
	\item $ k_4=\frac{\alpha^2c^4}{24}(\mu_g - \mu_p)^2 $.
\end{itemize}

According to \cite{Bonyadi2015Stagnation}, in order to guarantee the convergence of variance for IPSO $ <\omega,c,\alpha> $, the values of $ c $, $ \omega $, and $ \alpha $ need to guarantee conditions in Eq. \ref{Eq:PSOConvergenceConditions}. Satisfaction of Condition 2 in that equation for IPSO settings (i.e., $ c(1+\alpha) > 0 $) entails $ \alpha \ne -1 $. Also, $ c(1+\alpha) > 0 $ imposes that $ c $ and $ 1+\alpha $ always have the same sign, hence, $ c<0 $ if $ \alpha<-1 $ and $ c>0 $ if $ \alpha >-1 $.

We introduce $ V_c $ for IPSO, defined by $ V_x= \gamma V_c $, where
\begin{equation}
\gamma=[2(\alpha+1)^2[\sigma_p^2+\alpha^2\sigma_g^2]+\alpha^2(\mu_p-\mu_g)^2]
\label{Eq:gamma}
\end{equation}
and
\begin{equation}
V_c=\frac{-c(\omega + 1)}{c(m_2 - m_1\omega)+(\alpha+1)^3(6\omega^2 - 6)}
\label{Eq:IPSOCV1}
\end{equation}
\noindent where $ m_1=(\alpha + 1)^2(\alpha^2 + 3\alpha + 1) $ and $ m_2=(\alpha + 1)^2(2\alpha^2 + 3\alpha + 2) $. $ V_c $ was also introduced in \cite{Bonyadi2015Stagnation} but only for IPSO $ <\omega,c,1> $. The main difference between $ V_x $ and $ V_c $ is that $ V_c $ is independent of the position and variance of $ p $ and $ g $. Hence, one can set the range of the search through changing the value of $ V_c $ during the run. Of course the reflection of this change in $ V_x $ also depends on $ \gamma $. If this term is zero ($ \sigma_p^2=\sigma_g^2=0 $ and $ \mu_p=\mu_g $) then $ V_x $ is zero (particle stops moving) no matter the value of $ V_c $. However, while this term is non-zero, the variance of particles positions can be controlled by $ V_c $.

To achieve a desired $ V_c $, one can solve Eq. \ref{Eq:IPSOCV1} for $ c $ as: 
\begin{equation}
c=\frac{-6V_c(\alpha + 1)^3(\omega^2 - 1)}{V_c(m_2 - m_1\omega) + (\omega + 1)}
\end{equation}
We use this equation in section \ref{sec:searchability} to show how to control the movement patterns of particles in IPSO. Note that this is equivalent to Eq. 22 in \cite{BonyadiMovementPattern2017} when $ \alpha=1 $. 

\subsection{Focus of the search}
\label{sec:focussearch}
The focus of the search indicates to what extent the particle should concentrate the search around the best known solutions to the swarm, $ p $ and $ g $. We measure the search concentration around a point $ o $ (that is a random variable) as $ (E(x_t)-E(o))^2 $, where  $ x_t $s are the generated positions by the search algorithm. Hence, the concentration of the search around $ p $ and $ g $ can be defined by the average distance between $ p $ or $ g $ and the average of the positions. We introduce Theorem \ref{thr:focus} to enable controlling this focus.

\begin{convergenceTheorem}
	\label{thr:focus}
	We define $ F=(\frac{\mu_{\phi_2}}{\mu_{\phi_1}})^2 $. The particle positions are more concentrated around $ g_t $ if $ F>1 $ and more concentrated around $ p_t $ otherwise.
\end{convergenceTheorem}
\begin{proof}
The focus $ F $ can be formulated by the closeness of the expectation of movement to the expectation of $ p $ and $ g $, i.e. $ (E(x_t)-\mu_p)^2 $ and $ (E(x_t)-\mu_g)^2 $. In a long run, $ E(x_t) $ can be replaced by $ E_x $. We define $ F $, the focus measure, when $ t $ grows as 
\begin{equation}
F=\frac{(E_x-\mu_p)^2}{(E_x-\mu_g)^2}=\frac{(\frac{\mu_{\phi_1}\mu_p+\mu_{\phi_2}\mu_g}{\mu_{\phi_1}+\mu_{\phi_2}}-\mu_p)^2}{(\frac{\mu_{\phi_1}\mu_p+\mu_{\phi_2}\mu_g}{\mu_{\phi_1}+\mu_{\phi_2}}-\mu_g)^2}=
(\frac{\mu_{\phi_2}}{\mu_{\phi_1}})^2
\label{eq:searchfocus}
\end{equation}

If $ F>1 $ then the distance between $ E_x $ and $ \mu_g $ is smaller than the distance between $ E_x $ and $ \mu_p $, hence, the search is more concentrated around $ \mu_g $. If $ F<1 $ then the search is more concentrated around $ \mu_p $. For $ F=1 $ the search is balanced between $ p $ and $ g $.
\end{proof}

The measure $ F $ is independent of the inertia weight and distribution of $ p $ and $ g $ and it is only dependent on acceleration coefficients. For IPSO $ F=\alpha^2 $, i.e. $ |\alpha|>1 $ enforces the particle to focus the search more around the global best, as expected.

\subsection{Controlling movement patterns}
\label{sec:searchability}
The values of $ V_x $, $ \rho_i $ (for all $ i $), and $ F $ determine the pattern of movement in PSO. In order to achieve a desired pattern of movement, described by $ V_x $, $ \rho_i $, and $ F $, one can solve a system of equations that involves Eq. \ref{Eq:correlationGeneric} and \ref{Eq:varianceFixPoint} and \ref{eq:searchfocus}. Solving this system of equations is, however, not possible as it involves many unknown variables ($ \mu_\omega $, $ \mu_{\phi_1} $, $ \mu_{\phi_2} $, $ \sigma_{\phi_1} $, $ \sigma_{\phi_2} $, and $ \sigma_\omega $). Nevertheless, for IPSO and to guarantee achieving only $ \rho_1 $ and $ V_c $, this system of equations is simplified as follows:
\begin{eqnarray}
	\begin{cases}
		c & =\frac{2(1-\rho_1)(\omega+1)}{\alpha+1}\\
		c & =\frac{-6V_c(\alpha + 1)^3(\omega^2 - 1)}{V_c(m_2 - m_1\omega) + (\omega + 1)}
	\end{cases}
	\label{Eq:systemFC}
\end{eqnarray}
\noindent where  $ m_1=(\alpha + 1)^2(\alpha^2 + 3\alpha + 1) $ and $ m_2=(\alpha + 1)^2(2\alpha^2 + 3\alpha + 2) $. Theorem \ref{Thr:solvingtheSystemOfEq} is used to find the solutions for this system of equations \footnote{Note that $ V_x $ has been replaced by $ V_c $ in Eq. \ref{Eq:systemFC} so that the system of equations become independent of $ p $ and $ g $ and their changes. However, if it is desired to achieve a given $ V_x $ rather than $ V_c $, one can simply calculate a corresponding $ V_c $ for the given $ V_x $ (recall that $ V_x=\gamma V_c $) and then achieve the found $ V_c $. Of course this calculation involves finding the value for $ \gamma $.}. We first focus on finding $ \omega $ and $ c $ and then we provide details on finding proper $ \alpha $.

\begin{convergenceTheorem}
	\label{Thr:solvingtheSystemOfEq}
	For any $ V_c> 0 $ and $ \rho_1 \in (-1, 1) $, there exists a feasible solution $ (\omega_0, c_0) $ for the system of equations in Eq. \ref{Eq:systemFC} that guarantee convergence of variance (Eq. \ref{Eq:PSOConvergenceConditions}) where: 	
	\begin{eqnarray}
	\label{Eq:calculateOmegaC}
	(\omega_0,c_0)=\left(\frac{m_1V_c+m_2\rho_1 V_c+\rho_1-1}{m_2V_c+m_1 \rho_1 V_c-\rho_1+1},\right.\\
	 \left.\frac{2(1-\rho_1)(\omega_0+1)}{\alpha+1}\right)\nonumber
	\end{eqnarray}
\end{convergenceTheorem}
\begin{proof}
We introduce $ h_{\rho_1, V_c}(\omega) $ as follows:
\begin{eqnarray*}
\label{Eq:findingOmega}
h_{\rho_1, V_c}(\omega)=\frac{-6V_c(\alpha + 1)^3(\omega^2 - 1)}{V_c(m_2 - m_1\omega) + (\omega + 1)}-\frac{2(1-\rho_1)(\omega+1)}{\alpha+1}
\end{eqnarray*} 

The value of $ \omega $ in Eq. \ref{Eq:systemFC} can be calculated by solving $ h_{\rho, V_c}(\omega)=0 $. The roots of $ h_{\rho, V_c}(\omega) $ are $ \omega_0=\frac{m_1V_c+m_2\rho_1 V_c+\rho_1-1}{m_2V_c+m_1 \rho_1 V_c-\rho_1+1} $ and $ \omega_1=-1 $. As $ w_1=-1 $ leads to $ c=0 $ for any $ V_c $ and $ \rho_1 $ (to ensure first order stability, Eq. \ref{Eq:expectationConvergence}), we continue with $ \omega_0 $. $ \omega_0 $ leads us to a feasible solution for the system of equations in Eq. \ref{Eq:systemFC} if and only if its denominator is non-zero, i.e. $ m_2V_c+m_1 \rho_1 V_c-\rho_1+1\ne0 $. The value of $ m_2V_c+m_1 \rho_1 V_c-\rho_1+1 $ is zero if and only if $ \rho_1=-(m_2V_c + 1)/(m_1V_c - 1) $. The first derivative for this equation as a function of $ V_c $ is $ (m_1+m_2)/(m_1V_c - 1)^2 $ that is always positive that means $ -(m_2V_c + 1)/(m_1V_c - 1) $ is an increasing function of $ V_c $. Hence, as the value of this function for $ V_c=0 $ is 1, $ m_2V_c+m_1 \rho_1 V_c-\rho_1+1\ne0 $ for any $ V_c>0 $. This means that, for any $ V_c>0 $ and $ \rho_1 \in (-1, 1) $ and $ \alpha \in \mathbb{R} $, $ \omega_0 $ is a finite real value. 

The value of $ c_0 $ is calculated by substituting $ \omega_0 $ in the second equation in the system of equations:
\begin{equation*}
c_0=\frac{2(1-\rho_1)(\omega_0+1)}{\alpha+1}
\end{equation*}
Clearly, the denominator of $ c_0 $ is also non-zero for any $ V_c> 0 $ and $ \rho_1 \in (-1, 1) $ and $ \alpha \ne -1 $, hence, $ c_0 $ is always achievable.
\end{proof}

This theorem enables finding $ c_0 $ and $ \omega_0 $ to guarantee achieving a given $ \rho_1 $ and $ V_c $, hence, $ D_c $. The values for $ c_0 $ and $ \omega_0 $ are, however, dependent on $ \alpha $, a value that could be used to balance the focus of the search, $ F $. 

Fig. \ref{fig:varianceFreq} shows the calculation of $ \omega_0 $ and $ c_0 $ for four different correlations ($ \rho_1 \in \{-0.8, -0.1, 0.1, 0.8\} $) and variance coefficients ($ V_c \in \{0.1, 8.0, 0.15, 3.0\} $) when $ \alpha=1 $ (i.e., $ c_1=c_2 $). For each case, the curves corresponding to $ \rho_1 $ and $ V_c $ have been shown (in $ c $ vs $ \omega $ space). The crossing points of these curves are solutions to the system of equations presented in Eq. \ref{Eq:systemFC}.
\begin{figure}
	\centering
	\includegraphics[clip,width=0.5\textwidth]{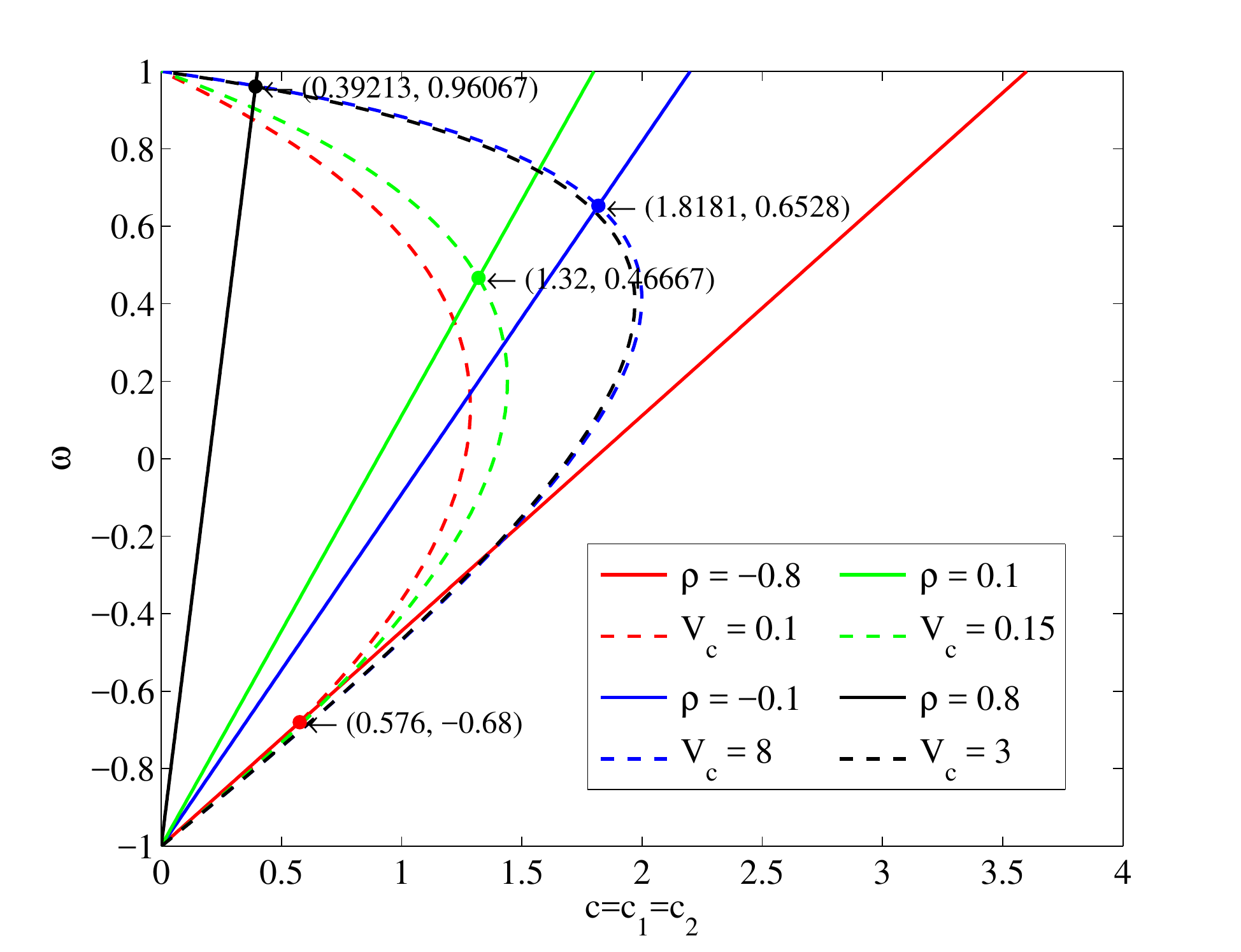}
	\caption{The solution to the system of equations for different values of $ V_c $ and $ \rho $, $ (V_c, \rho)=\{(0.1,-0.8), (8,-0.1), (0.15,0.1), (3,0.8)\} $.}
	\label{fig:varianceFreq}
\end{figure}

One can use $ \rho_1 $, $ F $, and $ V_c $ to set the movement pattern of a particle. Hence, if it is known which pattern of movement is more efficient to optimize the problem at hand\footnote{Determining appropriate values for variance and correlation to optimize the search efficiency is out of the scope of this article. Such decisions need to consider different factors such as the landscape type, difference from desired objective value, rate of improvements, etc. Once these are known, a decision inference system can determine whether a larger/smaller variance and correlation are needed. Interested readers are referred to \cite{malan2014characterising} for further information on this topic.} then Theorems \ref{thr:focus} and \ref{Thr:solvingtheSystemOfEq} could be used to calculate proper coefficients to achieve that pattern. The procedure proposed in \cite{BonyadiMovementPattern2017} could be used for a similar purpose, however, it involved solving a degree-4 polynomial to calculate $ \omega $ and $ c $ to achieve a given variance and base frequency (that is corresponding to $ \rho_1 $). Corresponding procedure proposed in this paper to achieve a given variance and $ \rho_1 $ does not apply any computational overhead to the main IPSO calculations as it is done in $ O(1) $. Also, the procedure proposed in \cite{BonyadiMovementPattern2017} was limited to $ c=c_1=c_2 $ and $ \omega > 0 $ while our proposed procedure supports any feasible value for these coefficients. 

\subsection{A time-adaptive PSO}
We propose and validate a new time-adaptive PSO based on the movement patterns analysis conducted in this paper, called the movement pattern adaptation PSO (MAPSO). We use the following key observations:
\begin{itemize}
	\item A large $ V_c $ and a small $ |\rho_1| $ is preferable at the early stages to enable the particle to search on a large range (large $ V_c $) and not towards any particular direction (small $ |\rho_1| $). $ F=1 $ would balance between $ p $ and $ g $, that would be a good choice at the beginning.
	\item As the iterations grow a larger $ |\rho_1| $ is beneficial to maintain good directions found by the particles. A large $ V_c $ is still beneficial as the exploration could be still helpful. 
	\item Later stages of the search would be better to focus on best found solutions (larger $ F $), around the best known solutions (smaller $ |\rho_1| $ and smaller $ V_c $).
\end{itemize}

Note that none of these patterns can be achieved by only changing the value of $ \omega $, but by changing all coefficients ($ \omega $, $ c $, and $ \alpha $) at the same time. We test this idea in the next section. Based on this setting, the values of $ V_c $, $ \rho_1 $, and $ F $ through:

\begin{equation*}
V_{c}^{(t)}=
\begin{cases}
V_{max} & t<t_1\\
\frac{(t-t_1)(V_{min}-V_{max})}{t_2-t_1}+V_{max} & t_1<t<t_2\\
V_{min} & t>t_2
\end{cases}
\end{equation*}
where $ V_{max}=25 $, $ V_{min}=5 $. This function ensures that the value of $ V_c $ is largest at the earlier and smallest at the later stages of the search, while linearly decreasing from iteration $ t_1 $ to $ t_2 $. 
\begin{equation*}
\rho_{1}^{(t)}=
\begin{cases}
\rho_{min} & t<t_1\\
\frac{(t-t_1)(\rho_{max}-\rho_{min})}{(t_2-t_1)/2-t_1}+\rho_{min} & t_1<t<(t_2-t_1)/2\\
\frac{(t-t_1)(\rho_{min}-\rho_{max})}{t_2-(t_2-t_1)/2}+\rho_{max} & (t_2-t_1)/2<t<t_2\\
\rho_{min} & t>t_2
\end{cases}
\end{equation*}
where $ \rho_{max}=0.8 $, $ \rho_{min}=0.1 $. The main rationale is to set $ \rho_1 $ to a small value at the early stages. The value is then grows until the mid stage of the search and starts declining afterwards. The equation ensures that the value of $ \rho_1 $ is small at the later stages of the search.
\begin{equation*}
F^{(t)}=
\begin{cases}
F_{min} & t<t_1\\
1 & t_1<t<t_2\\
F_{max} & t>t_2
\end{cases}
\end{equation*}
where $ F_{max}=25 $, $ F_{min}=.25 $. This equation ensures more concentration around the personal best, then a balanced search around both personal and global best, and finally focus the search around the global best. In all equations, we set $ t_1=\frac{t_{max}}{5} $, $ t_2=\frac{4t_{max}}{5} $. All of these values were set through some experiments on a very limited number of standard optimization functions\footnote{The aim of this article is not to find the best values for the coefficients to design yet another adaptive PSO to beat other existing PSO. The main aim of this article is to theoretically support considerations that need to be factored in for any adaptive PSO to be designed. Hence, the parameter setting here conducted in a very superficial level to show even such suboptimal settings can lead to good results.}. Figure \ref{fig:VandRhochanges} ($ t_{max}=10,000 $) demonstrates the value of coefficients in time according to these settings.
\begin{figure}
	\begin{tabular}{c}
		\includegraphics[width=0.45\textwidth]{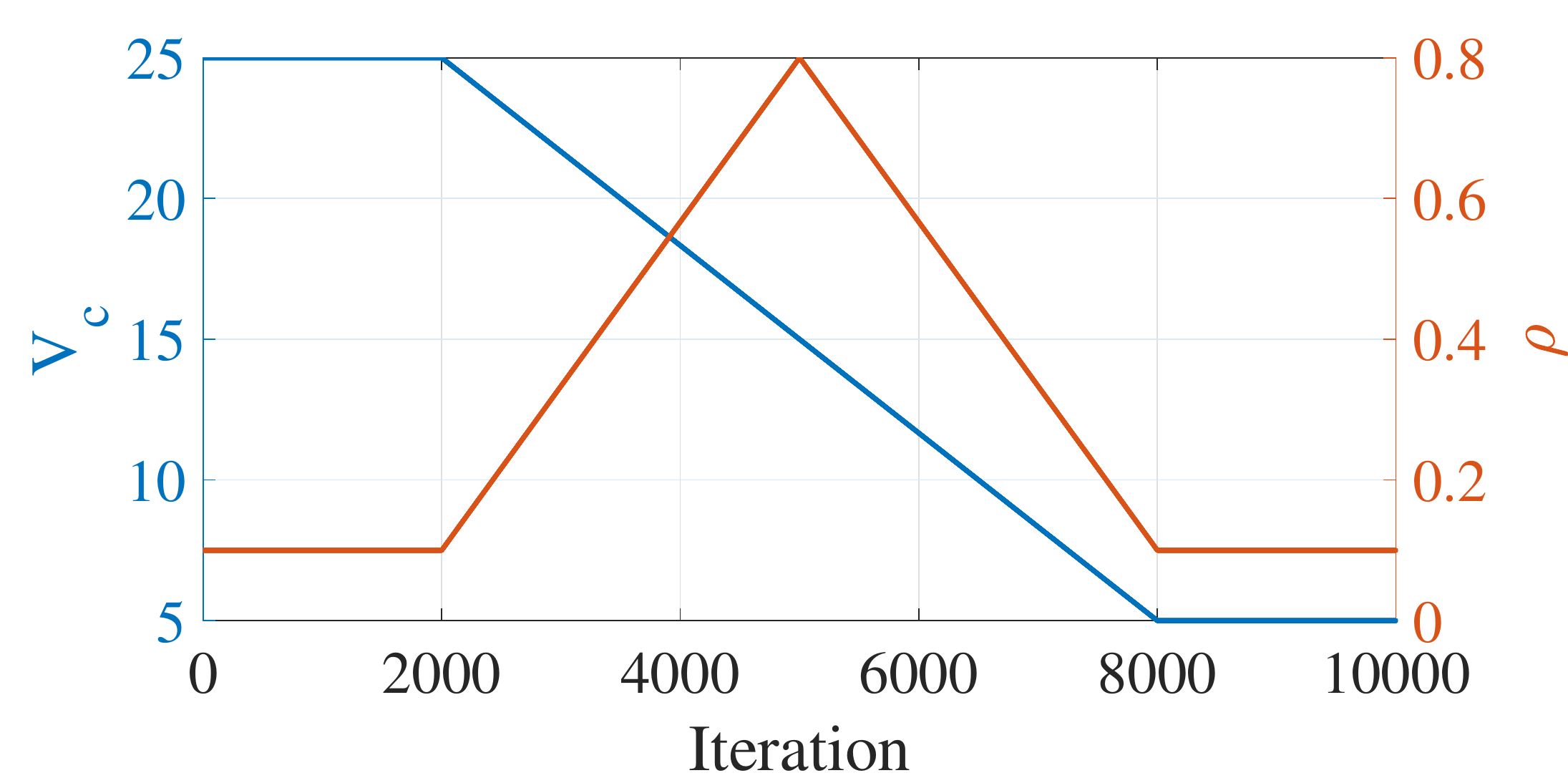}  \\
		(a) \\
		\includegraphics[width=0.45\textwidth]{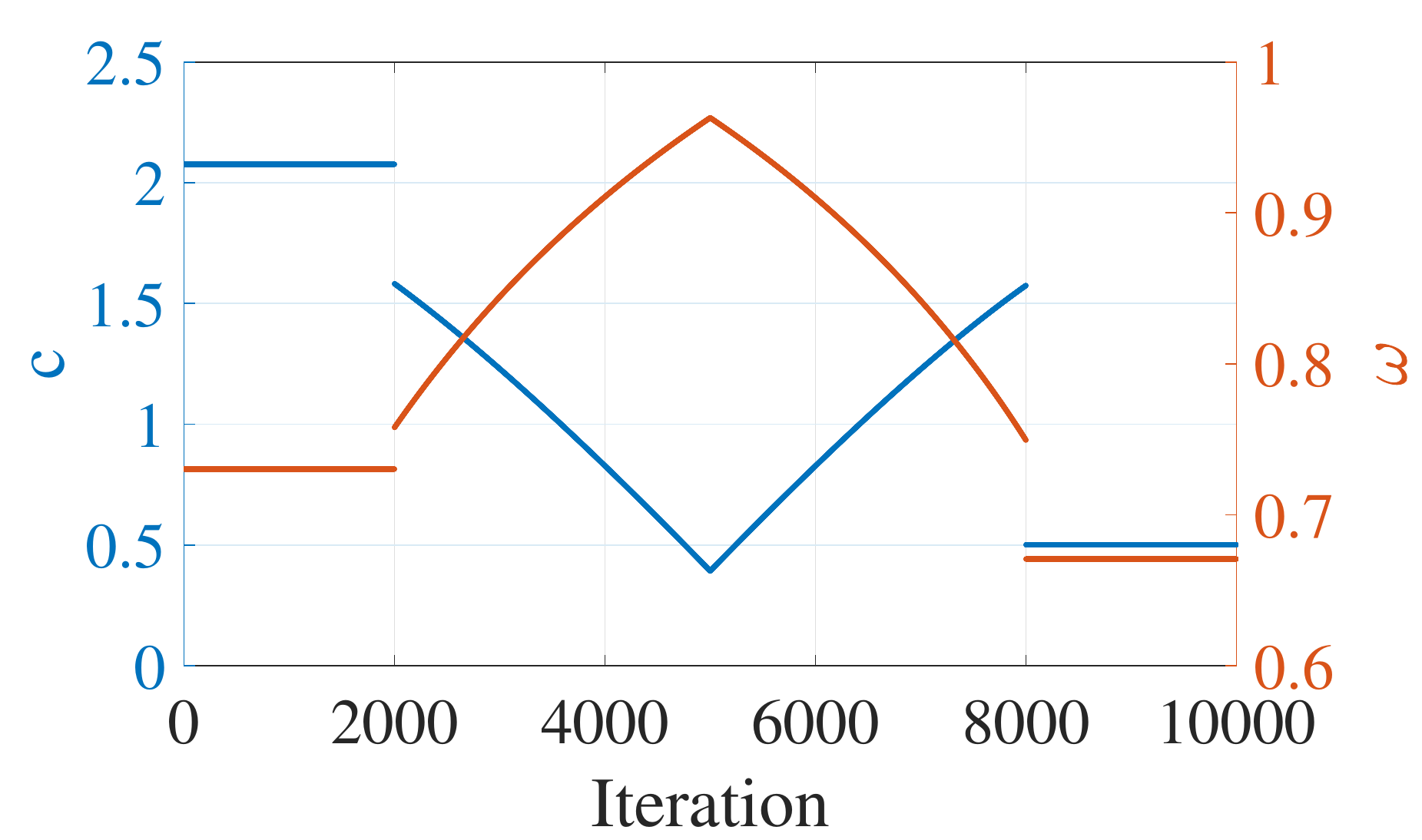} \\
		(b) \\
	\end{tabular}
	\caption{(a) Changes in movement patterns factors during iterations, where $ t_{max} $ was set to 10,000. (b) Corresponding values for $ \omega $ and $ c $.}
	\label{fig:VandRhochanges}
\end{figure}

\section{Experiments and comparisons}
In this section we compare 14 algorithms, 2 with constant coefficients, 12 with adaptive or time-adaptive coefficients, against a time adaptive approach based on analyses conducted in this paper. The methods for comparison are 
\begin{itemize}
	\item Top 9 methods in \cite{harrison2016inertia} that are: Constriction coefficient PSO (CCPSO) \cite{Clerc02Explo}, Linear decreasing inertia weight PSO (LDWPSO) \cite{Shi98Parameter}, Random inertia weight PSO (RWPSO) \cite{Eberhart01Tracking}, Chaotic descending inertia weight PSO \cite{feng2007chaotic}, sugeno inertia weight PSO \cite{lei2006new}, logarithm decreasing PSO \cite{gao2008particle}, self-regulating PSO \cite{tanweer2015self}, adaptive inertia weight PSO (AIWPSO) \cite{Nickabadi11NovelPSO}, adaptive velocity information PSO (AVIPSO) \cite{xu2013adaptive}. 
    \item Constant coefficients showed to be better than others in \cite{BonyadiMovementPattern2017}, inertia constant PSO (ICPSO).
	\item Linear increasing inertia weight PSO (LIWPSO) \cite{Zheng03EmpericalPSO}, 
	\item Decreasing inertia weight PSO (DWPSO) \cite{fan2007decreasing}, nonlinear improved inertia weight PSO (NLIPSO) \cite{jiao2008dynamic}, and nonlinear inertia weight PSO (NLPSO) \cite{yang2015low}.
	\item Proposed method (MAPSO).
\end{itemize}

\subsection{Comparison procedure}
We use a similar procedure to what was used in \cite{BonyadiMovementPattern2017} for comparison. Let $ A_{i,k} $ the set of all objective values over all runs found by the algorithm $ i \in \{1,...,n\} $ for the function $ k \in \{1,...,z\} $, i.e., $ A_{i,k} $ is a set of 50 values (50 runs each algorithm), each shows the objective value of the algorithm $ i $ at a particular run. We introduce the matrix $ C $ as follows:
\begin{eqnarray}
C_{i,j,k}= \begin{cases}
1 & \text{if }A_{i,k} \prec A_{j,k} \\
-1 & \text{if } A_{j,k} \prec A_{i,k}\\
0 & otherwise
\end{cases}
\nonumber
\end{eqnarray}

\noindent where $ A_{i,k} \prec A_{j,k} $ if and only if $ A_{i,k} $ is significantly better than $ A_{j,k} $ according to the Wilcoxon test (i.e., $ p<0.05 $ when $ A_{i,k} $ is compared with $ A_{j,k} $ and the median of $ A_{i,k} $ is smaller than $ A_{j,k} $'s). The value of $ T_{i,j}=\sum_{k=1}^z C_{i,j,k} $ indicates the number of functions for which the algorithm $ i $ was working significantly better than the algorithm $ j $\footnote{We assume that all functions are equally important, hence, if an algorithm performs better than another algorithm on more number functions then it is simply assumed that the first algorithm is better than the second. Thus, these calculations rely on the assumption that the selected benchmark represents the set of all problems of interest.}. If $ T_{i,j} > 0 $ (algorithm $ i $ beats algorithm $ j $) then the algorithm $ i $ is significantly better than algorithm $ j $ in more functions than $ j $ is significantly better than $ i $. The matrix $ T_{i,j} $ is used as the adjacency matrix to form the digraph $ G $. If $ T_{i,j}>0 $ then the edge $ (i,j) $ is added to $ G $. If $ T_{i,j}<0 $ then $ (j,i) $ is added to $ G $. The indegree of a node $ i $ in $ G $ would then indicate the number of algorithms the algorithm $ i $ can beat. 

\ignore{For each algorithm $ i $, we calculate the rank $ R_{i} $ as the number of times that $ T_{i,j} $ is positive for $ k =1 ... n $. If $ R_{i}=n-1 $ then the algorithm $ i $ is significantly better than all other algorithms in majority of test functions.}

\subsection{Experimental results}
The population size for all methods was set to $ 20 $ in all tests. The number of dimensions ($ D $) was $ 30 $, number of function evaluations was set to $ 5000D $, the benchmark test functions were taken from CEC2014 (30 functions) \cite{liang2013problem}, and number of runs was set to $ 50 $ to reduce the impact of initialization and randomization. Figure \ref{fig:alg-res} shows the digraph $ G $ when $ D=30 $. 
\begin{figure}
	\includegraphics[trim={3cm 4cm 3cm 2cm},clip,width=.5\textwidth]{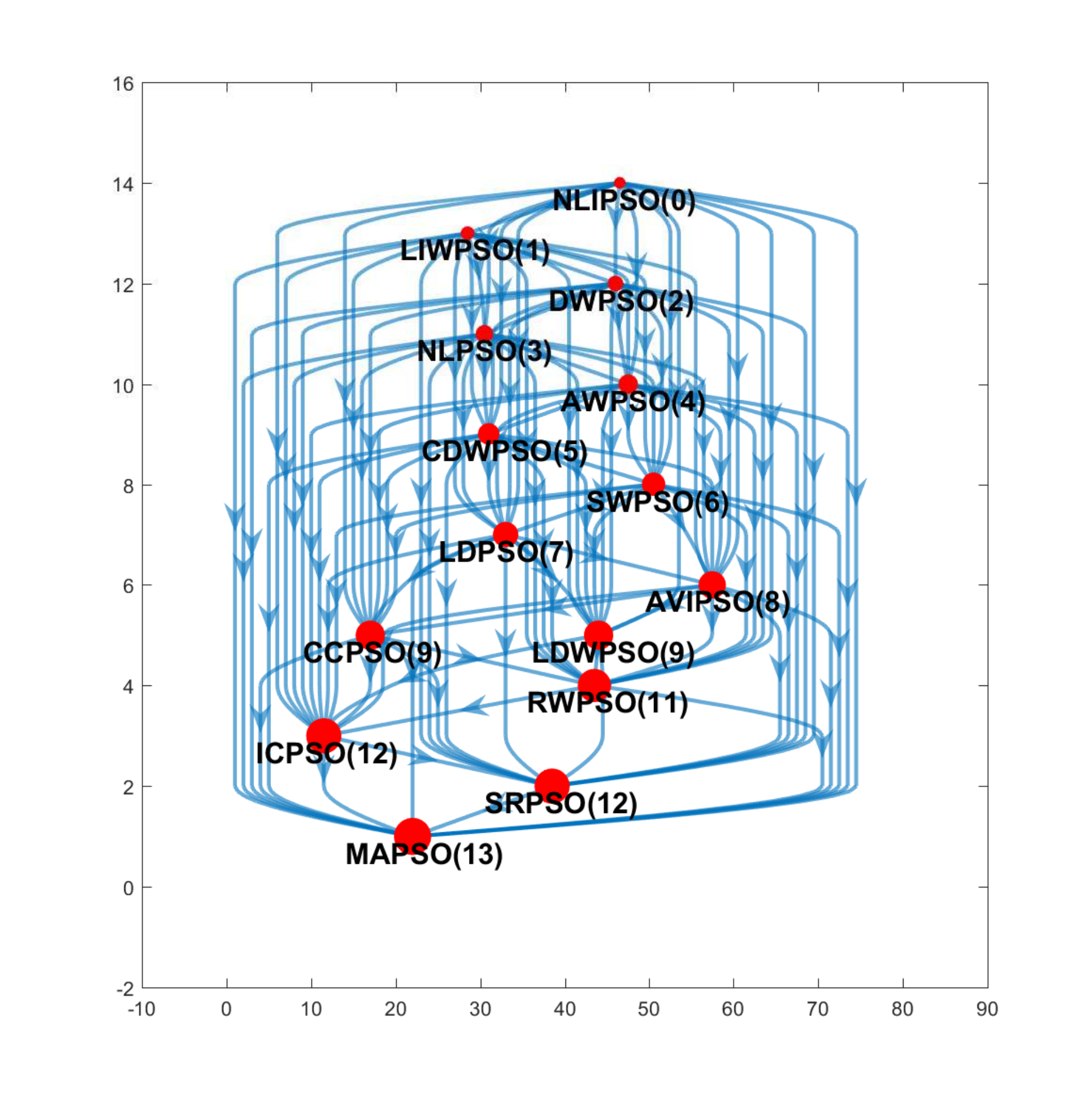} \\
	\caption{The values in the parentheses show the number of methods that a method can beat. An edge from $ i $ to $ j $ indicates that the method $ i $ beats the method $ j $. 30 dimensional problems with 20 particles in the swarms.}
	\label{fig:alg-res}
\end{figure}

From Figure \ref{fig:alg-res} it is clear that the simple time-adaptive approach designed based on the theoretical works in this paper performs well comparing to other existing  methods. This indeed indicates that adaptation of movement patterns provide more intuitive yet effective tool to design successful adaptive PSO methods.

\section{Conclusion and future works}
\label{sec:conclusions}
In this paper we investigated three factors, namely correlation between positions, expected movement distance, and focus of the search, that characterize movement patterns of particles in PSO. We formulated all of these factors as a function of particle coefficients (acceleration coefficients and inertia weight). All of our calculations were conducted under the most recent assumption for theoretical analysis of PSO (Assumption \ref{ass:genericRandomVariable}), i.e., all coefficients, personal best, and global best are random variables with predefined expectation and variance. Considering the particle position as a time series, we calculated the autocorrelation of the trajectory of the particle (the correlation between the particle position at iteration $ t+1 $ and its all previous positions) theoretically. We provided a guideline to control this linear dependency to achieve a random or smooth movement by the particle. We also provided details on the expected movement distance (e.g., expected velocity). We proved that the expected movement distance is a function of the expected search range, defined by the variance of movement, and the correlation between the particle position at each iteration with the previous. We discussed how the expected search range and expected movement distance are related to the means of global and local search abilities of particles. We also introduced a measure for the focus of the search that enables the user to control the particle ability in searching more around the personal or global best. Finally, we introduced a single equation that maps any feasible moment pattern, formulated by these three factors, to the inertia weight and acceleration coefficients, to ensure the given movement pattern is followed by the particle. We used these theoretical findings and designed a simple adaptive approach that was experimentally showed to be more effective than many existing adaptive PSO methods, validating our theoretical findings.

\ignore{
Our analyses showed that the movement patterns are not simple linear functions of the coefficients. Hence, a simple assumption that increasing inertia weight leads the particle to a better global search, made by \cite{Shi98Parameter}, is not correct as the expected search range, that is a measure for the particle ability to perform global or local movement, is in fact a non-linear function of all coefficients and not only the inertia weight. This assumption formed a base for many other time-adaptive PSO methods in which it was proposed to reduce the inertia weight during the run to ensure better local search at the latter stages of the search \cite{bonyadi2016particle}. Interestingly, \cite{zheng2003convergence} challenged this assumption and they showed that increasing inertia weight can be as effective as decreasing it during the search, an observation that was supported by some experiments. The main rationale in that article was that a small value for the inertia weight imposes a better "jump out" strategy while larger inertia weight leads the particle to settle. Although this is consistent with the means of autocorrelation in this paper, again, "jump out" is not only a function of autocorrelation but also the expected search range, i.e., a small expected search range might not be really effective for the "jump out" strategy.
}

The non-linear relationship between the movement patterns of particles and the coefficients, found in this article, indicate that the assumptions used for proposing adaptive approaches in PSO (e.g., \cite{Shi98Parameter,zheng2003convergence}) were somewhat simplistic. This provides theoretical justification for the findings in \cite{harrison2016inertia} where it was experimentally (on a rather large set of benchmark functions) shown that none of the PSO-based adaptive approaches tested in that study can beat a PSO with constant coefficients, that is indeed counterintuitive. 

We provided a simple approach to calculate coefficients (in $ O(1) $) in a way that the introduced movement characterization factors could be achieved. This provides a novel insight to the methods for controlling the particle coefficients, i.e., changing coefficients to achieve a particular pattern to perform a more effective search. Finding the relationship between the movement characteristics (autocorrelation and the expected search rang) with the characteristics of the landscape is one important future work. Ideally, a function should be designed that maps the search space characteristics (see \cite{malan2013survey} and \cite{malan2014characterising}) to $ <\rho_1, V_c, F> $. These values are then used to calculate $ <\omega,c,\alpha> $ through the procedure proposed in Theorem \ref{Thr:solvingtheSystemOfEq}.

\ignore{
\section{Acknowledgment}

}





\bibliographystyle{IEEEtran}
\bibliography{References}
\end{document}